\documentclass{article}
\usepackage{microtype}
\usepackage{graphicx}
\usepackage[font=small,labelfont=bf]{caption} 
\usepackage{subcaption}
\usepackage{booktabs} 

\usepackage[pagebackref=true]{hyperref}
\hypersetup{
    colorlinks=true,
    citecolor=darkblue,
    linkcolor=darkred,
    filecolor=darkblue,
    urlcolor=darkgreen,
}

\let\oldaddcontentsline\addcontentsline

\usepackage[preprint]{icml2026}

\makeatletter
\icml@noticeprintedtrue

\renewcommand{\printAffiliationsAndNotice}[1]{}
\def\icml@printAffiliationsAndNotice#1{}

\let\icml@affiliations\@empty
\let\icml@authorlist\@empty
\let\icmlcorrespondingauthor@text\@undefined
\def\Notice@String{}

\makeatother

\usepackage{amsmath}
\usepackage{amssymb}
\usepackage{mathtools}
\usepackage[capitalize,noabbrev]{cleveref}

\usepackage{microtype}
\usepackage{xurl}
\usepackage{booktabs}
\usepackage{nicefrac}
\usepackage[table]{xcolor} 
\usepackage{longtable}
\usepackage{array}
\usepackage{multirow}
\usepackage{enumitem}
\usepackage{wrapfig}
\usepackage{float}
\usepackage{makecell}
\usepackage{relsize}

\usepackage{amsmath, amssymb, mathtools, bm, amsfonts, amsthm}

\usepackage{booktabs}
\usepackage[font=small,labelfont=bf]{caption} 
\usepackage{algorithm}
\usepackage{algorithmic}

\usepackage{amsmath, amssymb, mathtools, bm}
\usepackage{amsfonts}     
\usepackage{amsthm}
\usepackage{capt-of}

\usepackage{amsmath,amsfonts,bm,bbm,mathrsfs}
\usepackage{amssymb,mathtools}

\def\eqref#1{equation~\ref{#1}}

\def\1{\bm{1}}

\DeclareMathAlphabet{\mathsfit}{\encodingdefault}{\sfdefault}{m}{sl}
\SetMathAlphabet{\mathsfit}{bold}{\encodingdefault}{\sfdefault}{bx}{n}

\providecommand{\KL}{D_{\mathrm{KL}}}

\usepackage{siunitx}
\newcommand{\mstd}[2]{#1{\scriptsize\,($\pm$#2)}}

\theoremstyle{plain}
\newtheorem{theorem}{Theorem}[section]
\newtheorem{proposition}[theorem]{Proposition}
\newtheorem{lemma}[theorem]{Lemma}
\newtheorem{corollary}[theorem]{Corollary}
\theoremstyle{definition}

\theoremstyle{remark}

\usepackage[most]{tcolorbox}
\usepackage{lipsum} 
\definecolor{lightgray}{gray}{0.9}
\usepackage{forest}
\usepackage{bm}

\definecolor{sota}{HTML}{F0FFDF}
\definecolor{royalblue}{HTML}{4169E1}

\newcommand{\tabcite}[1]{[\citenum{#1}]}
\usepackage[textsize=tiny]{todonotes}

\begin{document}

\twocolumn[
  \vspace{-10pt}
  \icmltitle{MolHIT: Advancing Molecular-Graph Generation with \texorpdfstring{\\}{ } Hierarchical Discrete Diffusion Models}
    
  \begin{center}
    {\bfseries Hojung Jung$^{1, \ast}$, Rodrigo Hormazabal$^2$, Jaehyeong Jo$^1$, Youngrok Park$^1$,} \\
    \vspace{5pt} 
    {\bfseries Kyunggeun Roh$^3$, Se-Young Yun$^1$, Sehui Han$^2$, Dae-Woong Jeong$^2$} \\
    \vspace{10pt}
    $^1$KAIST AI \quad $^2$LG AI Research \quad $^3$Seoul National University\\
    \vspace{5pt}
    \texttt{ghwjd7281@kaist.ac.kr}
  \end{center}

  \vskip 0.3in
]

\makeatletter
\def\blfootnote{\gdef\@thefnmark{}\@footnotetext}
\makeatother

\begin{NoHyper}
\blfootnote{\hspace{-1em}$^\ast$Work done during an internship at LG AI Research.}
\end{NoHyper}

\begin{abstract}
Molecular generation with diffusion models has emerged as a promising direction for AI-driven drug discovery and materials science. While graph diffusion models have been widely adopted due to the discrete nature of 2D molecular graphs, existing models suffer from low chemical validity and struggle to meet the desired properties compared to 1D modeling. In this work, we introduce \textbf{MolHIT}, a powerful molecular graph generation framework that overcomes long-standing performance limitations in existing methods. MolHIT is based on the Hierarchical Discrete Diffusion Model, which generalizes discrete diffusion to additional categories that encode chemical priors, and decoupled atom encoding that splits the atom types according to their chemical roles. Overall, MolHIT achieves new state-of-the-art performance on the MOSES dataset with near-perfect validity for the first time in graph diffusion, surpassing strong 1D baselines across multiple metrics. We further demonstrate strong performance in downstream tasks, including multi-property guided generation and scaffold extension. 
\end{abstract}

\section{Introduction}
\label{sec:introduction}

\begin{figure}[t!]
    \centering
    \includegraphics[width=\linewidth]{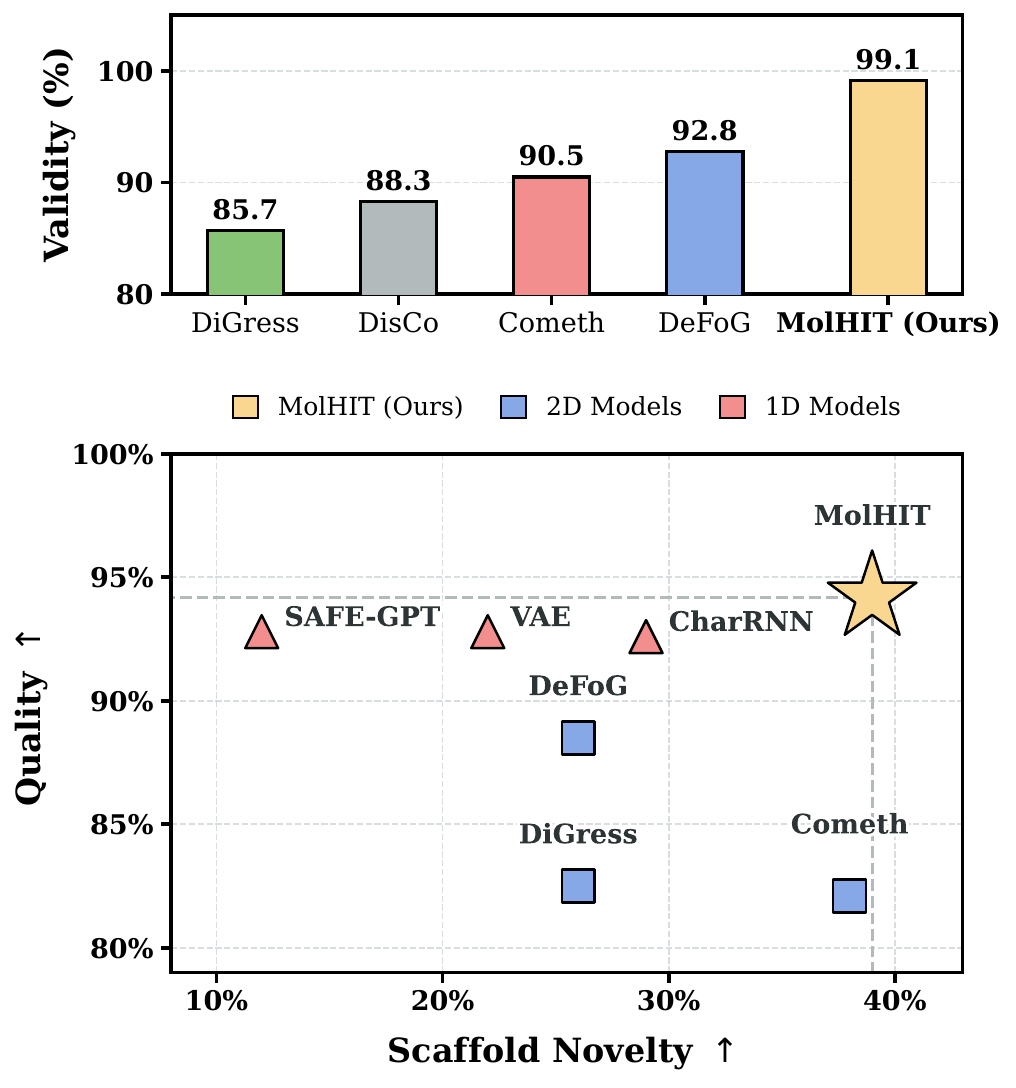}
    \vspace{-0.25in}
    \caption{MolHIT achieves SOTA result on MOSES dataset. (Top) Near-perfect validity, outperforming existing graph diffusion models. (Bottom) Pareto-optimal in quality-novelty trade-off.
    }
    \label{fig:frontpage_sota_results}
    \vspace{-0.25in}
\end{figure}

\begin{figure*}[t] 
    \centering
    \includegraphics[width=1.0\textwidth]{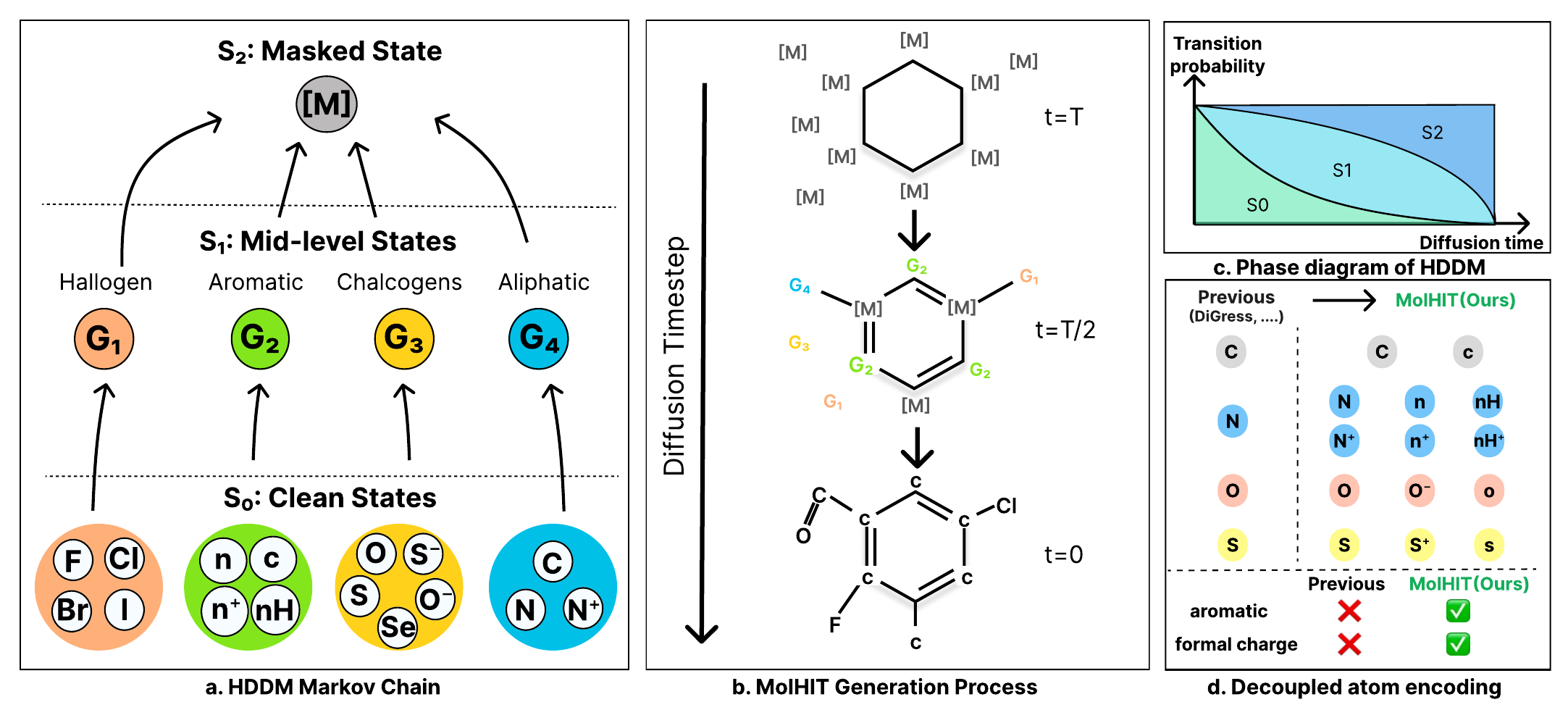}
    \vspace{-0.2in}
    \caption{\textbf{Overview of MolHIT}. (a) Markov chain of Hierarchical Discrete Diffusion Model (HDDM). Clean states ($S_0$) are transited to the mid-level states ($S_1$) and finally to the masked state ($S_2$). (b) Generation process of MolHIT. From the masked prior, atoms are denoised into mid-level states and then to atomic tokens in a coarse-to-fine manner. (c) Phase diagram of HDDM showing the transition probability of the forward process. (d) Decoupled atom encoding scheme, separately encoding the aromatic and charged atom types.
    }
    \label{fig:main_figure}
    \vspace{-0.1in}
\end{figure*}

Molecular generation with AI has the potential to significantly speed up materials design~\citep{sanchez2018inverse} and drug discovery~\citep{zhang2025artificial}. While this promise has led to many different modeling strategies, generating valid and novel molecules is challenging due to the vast combinatorial search space~\citep{dobson2004chemical}. Here, the primary challenge is not generating novel structures, but ensuring the structures remain chemically valid and relevant. Even a minor atom-level error can produce a structure that is chemically impossible or synthetically inaccessible. Consequently, it is necessary to develop generative models that efficiently explore this immense chemical space while generating valid and synthesizable molecules.

One common approach is to treat molecules as 1D sequences, 
most commonly through the SMILES representation~\citep{weininger1988smiles}. By representing molecular graphs into strings, these models can leverage powerful natural language processing techniques to learn patterns of characters. While this simpler learning objective results in generating valid molecules, they suffer from memorization, often reproducing patterns or common subsequences in the training set. This limited exploration capability creates a performance plateau as shown in Fig.~\ref{fig:frontpage_sota_results}, where high validity is achieved at the expense of a reduced number of new structures.

To overcome the exploration limits of sequence-based approaches, graph generative models~\citep{jo2022score, liu2023generative} treat molecules as interconnected systems of atoms and bonds. Unlike 1D models that often overfit to specific textual patterns, graph-based architectures are designed to internalize the underlying topological principles of chemical structures, allowing them to generalize beyond the training set and discover novel structures. In particular, discrete diffusion models~\citep{austin2021structured} have been widely studied for molecular graph generation as they naturally align with the categorical nature of atoms and bonds~\citep{vignac2022digress, xu2024discrete, qin2024defog, seo2025learning}.

While these models excel at structural exploration, they are prone to generating invalid or chemically unstable samples compared to well-optimized 1D models. This creates a performance gap that raises a fundamental research question: \textbf{Can we leverage the inductive biases of graph modeling to match the validity of sequence models while maintaining their superior capacity for structural novelty?}

We identify two critical limitations in existing molecular graph generation with discrete diffusion. \textbf{(1)} First, current uniform or absorbing transition treats each atom category as an independent category, even though there is well known chemical relationship that some atoms are easier to be replaced with another. Neglecting well-established domain priors often makes the learning unnecessarily hard, especially in molecular settings where high-quality molecule data is scarce.
\textbf{(2)} Second, existing graph models rely on naive atom encodings, ignoring the fact that a single atom can have different characteristics when it has a formal charge or consists of a ring (aromaticity). We reveal that this makes molecular graph generation tasks ill-posed and unnecessarily challenging, which we demonstrate in the reconstruction experiments in Fig.~\ref{fig:reconstruction_failure} where previous atom encoding fails.

In light of these observations, we introduce MolHIT, a hierarchical discrete diffusion framework designed to bridge the gap between structural innovation and chemical validity. Our framework is built upon the Hierarchical Discrete Diffusion Model (HDDM), where additional categories are added to represent natural chemical groups into the diffusion process. This coarse-to-fine approach allows the model to establish high-level chemical identities before refining them into specific atom types, thereby capturing the meaningful dependencies of molecular structure that uniform or absorbing kernels often overlook. Furthermore, we introduce Decoupled Atom Encoding (DAE) to resolve the information loss found in naive representations by explicitly split atoms based on their specific chemical roles, such as formal charge and aromaticity. By providing a chemical role into each token, DAE not only resolves the reconstruction problem in original atom encoding, but also reduces the burden of differentiating atom roles solely with the ($O(n^2)$) bond features. Combined together, \textbf{MolHIT} reaches a new Pareto frontier in generating novel structures with high quality, surpassing both existing 1D and 2D models (Fig.~\ref{fig:frontpage_sota_results}).

We extensively evaluate MolHIT with experiments on large molecular benchmarks, including unconditional generation tasks on MOSES~\citep{polykovskiy2020molecular} and GuacaMol~\citep{brown2019guacamol} benchmarks and conditional generation tasks, including scaffold extension and multi-property guided generation tasks. Across all benchmarks and tasks, MolHIT shows significant improvements over previous graph diffusion models, resulting in a new state-of-the-art that surpass 1D models.

Our contributions can be summarized as follows:
\vspace{-0.1in}
\begin{itemize}[itemsep=0.5mm, parsep=3pt]
    \item We introduce \textbf{MolHIT}, a molecular graph diffusion model built upon a novel \textbf{Hierarchical Discrete Diffusion Model (HDDM)} framework with a mathematically guaranteed ELBO.
    \item We identify a critical limitation in the prior graph generative models' atom encoding and propose a simple solution: \textbf{Decoupled Atom Encoding (DAE)}. By representing atoms based on their specific chemical roles, we find DAE enhances both the model’s generative expressiveness and chemical reliability.
    \item We achieve the SOTA performance on the MOSES benchmarks in multiple metrics, significantly outperforming both existing graph diffusion models and 1D sequence-level baselines.
    \item We test our algorithm on practical downstream tasks including multi-property guided generation and scaffold extension, achieving the highest performance compared to the previous graph diffusion approach.
\end{itemize}
\section{Preliminaries}
\label{sec:preliminaries}

\begin{figure}[t]
    \centering
    \includegraphics[width=\linewidth]{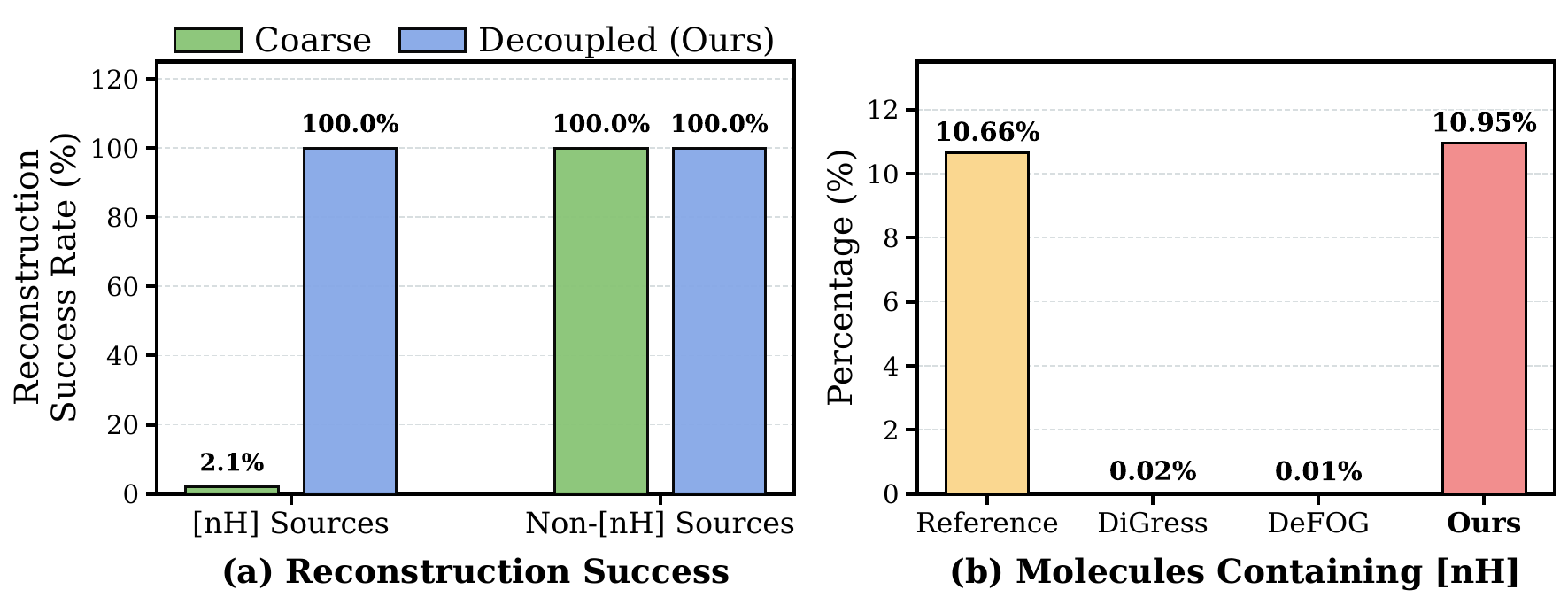}
    \vspace{-0.25in}
    \caption{Existing atom encoding for molecular graph is ill-posed. (Left) Reconstruction success rate on the Moses dataset with previous encoding and our decoupled atom encoding. (Right) Proportion of generated molecules containing pyrrolic nitrogen $[nH]$.}
    \label{fig:reconstruction_failure}
    \vspace{-0.15in}
\end{figure}

\subsection{Discrete Diffusion Models}
Given a discrete state space $S$ with $K$ categories, discrete diffusion models define a noising and denoising process within a discrete space. Specifically, for $\mathbf{x}\in S$, the noising process is described by a Markov chain as follows:
\begin{equation}
    q(\mathbf{x}_t|\mathbf{x}_{t-1})=\mathrm{Cat}(\mathbf{x}_t;\mathbf{x}_{t-1}Q_t).
\end{equation}
Here, marginal probability of $\mathbf{x}_t$ in timestep $t$, given clean data $\mathbf{x}_0$ can be calculated with 
\begin{equation}
    q(\mathbf{x}_t|\mathbf{x}_0)=\mathrm{Cat}(\mathbf{x}_t;\mathbf{x}_{0}\bar{Q}_t), \;\; \bar{Q}_t=Q_tQ_{t-1}\cdots Q_1.
\end{equation}
As shown by \citet{austin2021structured}, one can design multiple types of diffusion process, where two types of processes are widely used because of the closed-form calculation of the forward process and natural noising process.

\paragraph{Uniform transition}
Uniform transition assumes uniform prior $p_T(\mathbf{x}_T = c) = \frac{1}{K}$ $\text{for all } c \in \{1, \dots, K\}$. Then one could define a forward noising process by interpolating clean data $\mathbf{x}_0$ with the prior in the following way:
\begin{equation}
q(\mathbf{x}_t|\mathbf{x}_0) = \mathrm{Cat}(\mathbf{x}_t; (1-\bar{\alpha}_t)\frac{1}{K}\mathbf{1}\mathbf{1}^{T} + \bar{\alpha}_t \mathbf{x}_0),
\end{equation}
where $\bar{\alpha}_t$ is monotonic decreasing function with $\bar{\alpha}_0]=1, \bar{\alpha}_T=0$, which we call diffusion scheduler.

\paragraph{Marginal transition}
To facilitate the diffusion learning, marginal transition assumes data prior $\pi$ to be an optimal probability distribution that approximates the empirical data distribution from the training set. This has been primarily adopted for graph diffusion models DiGress~\citep{vignac2022digress, siraudin2024cometh}, where further details are in Appendix~\ref{app:further_backgrounds_marginal_transition}.

\paragraph{Absorbing transition}
Unlike uniform and marginal transition where diffusion process operates on the given K categories, one can introduce an additional masked (absorbing) state $\mathbf{m}$ with prior $\mathbf{e_m}$ being a one-hot vector of the masked state. Then, one can naturally define a diffusion process as an absorbing process in a Markov chain, which results in the following forward form:
\begin{equation}
    q(\mathbf{x}_t|\mathbf{x}_0) = \mathrm{Cat}(\mathbf{x}_t; \bar{\alpha}_t \mathbf{x}_0 + (1-\bar{\alpha}_t)\mathbf{e_m}).
\end{equation}
Given $q(\mathbf{x}_{t-1}|\mathbf{x}_t,\mathbf{x}_0) = \mathrm{Cat}(\mathbf{x}_{t-1};\frac{\mathbf{x}_tQ_{t|s}^\top\;\odot\;\mathbf{x}_0\bar{Q}_{t-1}}{\mathbf{x}_0\bar{Q}_t\mathbf{x}_t^T})$, one could estimate posterior $p_{\theta}(\mathbf{x}_{t-1}|\mathbf{x}_t,\mathbf{x}_0)$ by learning to estimate the clean data $\hat{\mathbf{x}}_0$ given the noisy data $\mathbf{x}_t$. This enables training the diffusion models with simple cross-entropy loss, where the loss function becomes directly linked to the negative evidence lower bound (NELBO)~\citep{austin2021structured,sahoo2024simple}.

\subsection{Molecular Graph Generation with Discrete Diffusion}
\label{subsec:3.2}

Given molecular graph $G=(X,E)$, denote $X\in\mathbb{R}^{n\times d_X}$, $E\in\mathbb{R}^{n\times n\times d_E}$ for the atom matrix and adjacency matrix (bond matrix) where $n$ is the number of atoms and $d_X, d_E$ are feature dimensions of atoms and edges. The forward process of discrete diffusion operates independently on the atom and bond matrices:
\begin{equation}
\begin{aligned}
    & G_t=(\mathbf{X}_t,\mathbf{E}_t) \;\;: \;\mathbf{X}_t = \mathbf{X}_0\mathbf{\bar{Q}}_{\mathbf{X},t},\;
     \mathbf{E}_t = \mathbf{E}_0\mathbf{\bar{Q}}_{\mathbf{E},t}.\\
\end{aligned}
\end{equation}
where, we define $\bar{Q}_{X,t}=Q_{X,t}\cdots Q_{X,1}, \bar{Q}_{E,t}=Q_{E,t}\cdots Q_{E,1}$ are forward transition matrix usually calculated in a closed form for efficiency.

Given a noisy graph $G_t$, a neural network is trained to estimate a clean graph $G_0=(X_0,E_0)$ through predicting clean atoms and bonds independently, which in practice results in the following cross-entropy (CE) loss:
\begin{equation}
\label{eq:trainig_loss_digress}
\begin{aligned}
&\mathcal{L}_{\theta}
=
\mathbb{E}_{t,\,G_t\sim q(\cdot\mid G_0)}
\Bigg[
\sum_{i=1}^{n}
-
\log p_{\theta}^{X}\!\left(X_{0,i}\mid G_t,t\right)
\\ 
& +
\lambda\!\!\sum_{1\le i<j\le n}
-\log p_{\theta}^{E}\!\left(E_{0,ij}\mid G_t,t\right)
\Bigg],
\end{aligned}
\end{equation}
where $\lambda > 0$ is a weighting factor that balances the relative contribution of node and edge loss.
\section{MolHIT Framework}
\label{sec:method}

\subsection{Hierarchical Discrete Diffusion Models}

We introduce \emph{Hierarchical Discrete Diffusion Models} (HDDM), which generalize the discrete diffusion framework into a multi-stage setting. Unlike standard discrete diffusion~\citep{austin2021structured}, where the forward transitions operate either within the clean vocabulary space or toward an absorbing (masked) state, HDDM introduces additional mid-level states that bridge the corruption process. 

To design a discrete diffusion in this augmented space, we first show that there exists a simple forward process that admits a tractable closed-form transition kernel. Specifically, for clean state space $\mathcal{S}_0$ with $K$ categories, suppose we add additional $G+1$ categories such that we have an augmented state space $\mathcal{T}$ with cardinality $D = K + G + 1$. As illustrated in Figure~\ref{fig:main_figure}-(a), we partition $\mathcal{T}$ into three disjoint subsets: $\mathcal{S}_0$, mid-level states $\mathcal{S}_1$ with $G$ categories, and the masked state $\mathcal{S}_2 = \{m\}$. 

Now, we define the transition kernel via a row-stochastic matrix $\mathbf{\Phi} \in [0, 1]^{K \times G}$, where $\mathbf{\Phi}_{ij}$ represents the probability of mapping any element $i \in \mathcal{S}_0$ to a mid-level element $j \in \mathcal{S}_1$. This operator induces a transition matrix $Q^{(1)} \in [0, 1]^{D \times D}$ on the full space $\mathcal{T}$, structured as a block matrix relative to the partition $(\mathcal{S}_0, \mathcal{S}_1, \mathcal{S}_2)$:$$Q^{(1)} = 
\begin{bmatrix}
\mathbf{0}_{K \times K} & \mathbf{\Phi} & \mathbf{0}_{K \times 1} \\
\mathbf{0}_{G \times K} & \mathbf{I}_{G} & \mathbf{0}_{G \times 1} \\
\mathbf{0}_{1 \times K} & \mathbf{0}_{1 \times G} & 1
\end{bmatrix}$$
Here, $\mathbf{I}_G$ is the identity matrix, indicating that states in $\mathcal{S}_1$ are absorbing the clean states in $\mathcal{S}_0$ under $Q^{(1)}$. Similarly, we define the masking operation via the transition matrix $Q^{(2)}$, which maps all states in $\mathcal{S}_0 \cup \mathcal{S}_1$ to the unique absorbing state in $\mathcal{S}_2$:$$Q^{(2)} = 
\begin{bmatrix}
\mathbf{0}_{(K+G) \times (K+G)} & \mathbf{1}_{(K+G) \times 1} \\
\mathbf{0}_{1 \times (K+G)} & 1
\end{bmatrix}$$
These transition matrices form the basis of the HDDM forward process as in the following lemma:

\begin{tcolorbox}[
    sharp corners,                        
    colback=white,                        
    colframe=black!50,                     
    coltitle=black,                        
    fonttitle=\bfseries,                   
    attach title to upper,                  
    after title={\smallskip\hrule\smallskip}, 
    left=5pt, right=5pt, top=5pt, bottom=5pt 
]
\begin{lemma}
\label{lemma:HDDM_forward_closed_form}
Define diffusion schedules $\alpha_t, \beta_t$ to be monotonically decreasing functions satisfying the boundary conditions $\alpha_0=\beta_0=1$ and $\alpha_1=\beta_1=0$, such that $\alpha_t \leq \beta_t$ for all $t$. We define the forward diffusion process of the hierarchical Markov chain via the transition kernel $Q_{t|s}$ from timestep $s$ to $t$ as:
\begin{equation}
Q_{t|s} = \alpha_{t|s} \mathrm{I} + (\beta_{t|s}-\alpha_{t|s})Q^{(1)} + (1-\beta_{t|s})Q^{(2)},
\end{equation}
where $\alpha_{t|s}:=\alpha_t / \alpha_s, \beta_{t|s}:= \beta_t / \beta_s$.
Then, the transition kernels satisfy the Chapman–Kolmogorov equation, such that $Q_{t|s}Q_{s|r} = Q_{t|r}$ for any $r < s < t$. Consequently, the cumulative forward transition from the initial state to timestep $t$ is given by:
\begin{equation}
\label{eq:forward_process}
Q_t = \alpha_t \mathrm{I} + (\beta_t-\alpha_t)Q^{(1)} + (1-\beta_{t})Q^{(2)},
\end{equation}
where $\mathrm{I}$ denotes the identity matrix in $\mathcal{T}$.
\end{lemma}
\end{tcolorbox}

Note that the above forward transition operators can be naturally extended to arbitrary hierarchies in state space. We provide a proof of the above lemma with a generalized forward process in Appendix~\ref{app:proof_of_closdeform_lemma}. 

Now for training guarantee, one can derive negative ELBO (NELBO), which we prove in Theorem~\ref{theorem:general_HDDM_NELBO} in Appendix. 
In practice, one can define $\mathbf{\Phi}$ as a deterministic projection that clusters clean atom categories into meaningful groups. We show in this special case, NELBO of HDDM can be further simplified as in the following.

\begin{tcolorbox}[
    sharp corners,                         
    colback=white,                         
    colframe=black!50,                     
    coltitle=black,                        
    fonttitle=\bfseries,                  
    attach title to upper,                
    after title={\smallskip\hrule\smallskip}, 
    left=5pt, right=5pt, top=5pt, bottom=5pt
]
\begin{theorem}
\label{theorem: continuous time NELBO}
If the forward transition kernels $Q_t$ in Eq.~\ref{eq:forward_process} is induced from the deterministic projection $\mathbf{\Phi}$, the continuous time NELBO of HDDM is given as:
\begin{equation}
\label{eq:NEBLO_deterministic_mainbody}
\begin{aligned}    
& \mathcal{L}_{\text{NELBO}}^{\infty}(\theta) = \\
&\mathbb{E}_{Q,t}\frac{\alpha_t(\beta_{t}'-\alpha_t')}{\beta_t-\alpha_t}\log{\langle\mathbf{x}_{\theta},\mathbf{x}\rangle}\cdot\mathbb{I}\left[\mathbf{z}_t\in \mathcal{S}_1\right] +
\\ & 
\frac{\beta_t'(\beta_t - \alpha_t)}{\beta_t(1-\beta_t)}\log\langle Q^{(1)}\mathbf{x}_{\theta}, Q^{(1)}\mathbf{x})\rangle \cdot\mathbb{I}\left[\mathbf{z}_t=\mathbf{m}\right]
\\ &
+\frac{\alpha_t\beta_t'}{\beta_t(1-\beta_t)}\log{\langle\mathbf{x}_{\theta},\mathbf{x}\rangle}\cdot\mathbb{I}\left[\mathbf{z}_t=\mathbf{m}\right] + C,
\end{aligned}
\end{equation}
for some constant $C$ and the denoiser $\mathbf{x}_{\theta}(\mathbf{z}_t, t)$.
\end{theorem}
\end{tcolorbox}

We provide a proof in Appendix~\ref{app:proof_NELBO}. For a sanity check, one can observe that Eq.~\ref{eq:NEBLO_deterministic_mainbody} reduces to the NELBO of the original masked diffusion models when $\beta_t=\alpha_t$ (i.e, no $\mathcal{S}_1$). With Theorem~\ref{theorem: continuous time NELBO}, we can design a simple cross-entropy loss for HDDM training in a principled way. We empirically find that regularization loss in Eq.~\ref{eq:NEBLO_deterministic_mainbody} does not improve the performance, so we take the original loss in Eq.~\ref{eq:trainig_loss_digress}.

\subsection{Decoupled Atom Encoding}
\label{subsec:method_decoupled_atom_encoding}

Existing graph diffusion frameworks~\citep{vignac2022digress, xu2024discrete, qin2024defog} typically rely on a coarse atom encoding scheme, where node identities are determined solely by their atomic numbers. While this simplifies the encoding, we identify that this one-to-many mapping between atomic tokens and their physical states (e.g., protonation or aromaticity) causes the generative task to be ill-posed. As illustrated in Fig.~\ref{fig:reconstruction_failure} (Left), this leads to a systematic reconstruction failure in molecules requiring fine-grained atomic descriptors, such as specific nitrogen motifs found in drug-like scaffolds. Consequently, models using these coarse encodings suffer from a representational bias, struggling to generate essential motifs that are statistically prevalent in the training distribution (Fig.~\ref{fig:reconstruction_failure}, Right).

To resolve these representational gaps and ensure the model can generalize across diverse chemical spaces, we introduce Decoupled Atom Encoding (DAE). DAE expands atomic state space by explicitly encoding aromaticity and formal charge as primary node attributes. This results in a near-perfect reconstruction ratio both on the MOSES and GuacaMol dataset.
Furthermore, by providing the model with necessary structural priors, \textbf{MolHIT} successfully recovers the distribution of complex motifs such as pyrrolic nitrogen ([nH]), which baselines using coarse encoding struggle to capture (Fig.~\ref{fig:reconstruction_failure}, Right). Further details are in Appendix~\ref{app:DAE}.

\subsection{Forward and Reverse Process of MolHIT}

\paragraph{Forward process of MolHIT}
Since the atom and bond are perturbed independently throughout the forward process, we decouple their transition dynamics in graph diffusion. This flexibility is particularly advantageous for molecular graph modeling. We empirically observe that a uniform transition kernel is essential for edge generation, whereas HDDM yields superior performance for atom types compared to a uniform approach. Therefore, we employ an HDDM process for atoms and a uniform transition for edges, resulting in the following forward process dynamics:

\begin{align}
    & {Q}_{X,t}= {\alpha}_{X,t}\mathrm{I} + (\beta_{X,t}-\alpha_{X,t}){Q}_{X,t}^{(1)} + (1- \beta_{X,t}){Q}_{X,t}^{(2)}, \notag \\
    & {Q}_{E,t}= {\alpha}_{E,t}\mathrm{I} + (1-\alpha_{E,t})\mathbf{1}_{d_E}\mathbf{1}_{d_E}^{T}, 
\end{align}
Our preliminary experiments show robustness on the HDDM scheduler $\alpha_{X,t}$, $\beta_{X,t}$, and therefore we simply opt for linear schedule for $\alpha_{X,t}=\alpha_{E,t}=1-t$ and $\beta_{X,t}=1-t^2$ for the experiments.

\paragraph{Grouping strategy}
Given the mid-level states $\mathcal{S}_1$, HDDM allows for the design of arbitrary transition kernels from $\mathcal{S}_0$ to $\mathcal{S}_1$. We implement a deterministic grouping kernel that clusters atom elements based on their intrinsic chemical properties and aromaticity. For instance, in the MOSES dataset, we partition 12 atom types into four semantic groups: $\{C\}$, $\{N, O, S\}$, $\{F, Cl, Br\}$, and $\{c, o, n, nH, s\}$. This hierarchical structure simplifies the initial stages of diffusion by focusing on broad chemical categories before refining specific identities. We extend this strategy to other datasets, such as GuacaMol, by adapting the groupings to their respective atom vocabularies. Full details of these partitions are provided in Appendix~\ref{app:grouping_HDDM_details}


\begin{figure}[t]
\vspace{-0.15in}
\centering
\begin{minipage}{1.0\linewidth}
\renewcommand{\baselinestretch}{1.1}
\begin{algorithm}[H]
\caption{PN-sampler with temperature sampling}
\label{alg:pn_sampler}
\begin{algorithmic}[1]
\STATE \textbf{Input:} Sample size $S$, Timesteps $T$, Temperature $\tau$, Nucleus threshold $p$
\FOR{$i = 1$ to $S$}
    \STATE Sample $N \sim P_{\text{train}}(N)$
    \STATE $G_T \sim p_T(G_T)$ \COMMENT{$G = (\mathbf{X}, \mathbf{E})$}
    \FOR{$t = T$ down to $\Delta t$ with step $\Delta t$}
        \STATE $\hat{p}_0(\mathbf{X}), \hat{p}_0(\mathbf{E}) \leftarrow f_\theta(G_t, t)$
        \STATE $\hat{p}'_0(\mathbf{X}) \leftarrow \text{TopP}(\text{Softmax}(\hat{p}_0(\mathbf{X}) / \tau), p)$ 
        \STATE $\hat{\mathbf{X}}_0 \sim \text{Categorical}(\hat{p}'_0(\mathbf{X}))$
        \STATE $\hat{\mathbf{E}}_0 \sim \text{Categorical}(\hat{p}_0(\mathbf{E}))$
        \STATE $G_{t-\Delta t} \sim q(G_{t-\Delta t} | \hat{G}_0)$ where $\hat{G}_0 = (\hat{\mathbf{X}}_0, \hat{\mathbf{E}}_0)$
    \ENDFOR
    \STATE \textbf{Store} $G_{\text{final}}$
\ENDFOR
\end{algorithmic}
\end{algorithm}
\end{minipage}
\vspace{-0.15in}
\end{figure}

\paragraph{Project and Noise (PN-sampler)}
Due to the standard ELBO guarantee as we prove in Theorem~\ref{theorem: continuous time NELBO}, one can sample from the original posterior update as in prior works~\citep{austin2021structured}. 
While standard posterior updates follow the transition $q(G_{t-\Delta t} | G_t, G_0)$ as justified by the ELBO guarantee in Theorem~\ref{theorem: continuous time NELBO}, we empirically find that this approach often restricts the structural exploration necessary for complex molecular generation. To address this, we design a Project-and-Noise (PN) sampler. PN sampler projects model's denoising prediction $p_{\theta}(G_0|G_t)$ onto the clean manifold $\mathcal{M}$ (one-hot vector) via categorical sampling to obtain a discrete candidate $\hat{G}_0$. This candidate is then directly re-noised to the preceding timestep $s = t - \Delta t$ using the cumulative transition kernel $Q_t$, effectively bypassing posterior constraints of $G_t$ to encourage greater diversity in the generated graph. The overall algorithm is illustrated in Alg.~\ref{alg:pn_sampler}.

\begin{table*}[t!]
\centering
\caption{Comprehensive MOSES benchmark results. Scaffold Novelty (Scaf-Novel) measures the ratio of novel scaffold molecules to the number of generated molecules, while Scaffold Retrieval (Scaf-Ret.) quantifies test scaffold retrievals. All of the results are the averaged value over 3 runs of 25,000 samples. Bold denotes the best in each category, and underline indicates SOTA performance within the 2D Graph models. Empty values are due to the absence of publicly available checkpoints or samples.
}
\label{tab:moses_comprehensive}
\resizebox{\textwidth}{!}{
\renewcommand{\arraystretch}{0.98}
\renewcommand{\tabcolsep}{5pt}
\small
\begin{tabular}{l l | c | c c  | c c c c c c c}
\toprule
Category & Model  & Quality $\uparrow$ & Scaf-Novel $\uparrow$ &Scaf-Ret. $\uparrow$ & Valid $\uparrow$ & Unique $\uparrow$ & Novel $\uparrow$ & Filters $\uparrow$ & FCD $\downarrow$ & SNN $\uparrow$ & Scaf $\uparrow$ \\

\midrule
- & Training set &  95.4 & \textemdash & \textemdash & 100.0 & 100.0 & \textemdash & 100.0 & 0.48 & 0.59 & - \\

\midrule
1D Sequence & VAE\tabcite{kingma2013auto} 
& 92.8 & 0.22 & 0.031
& 97.7 & 99.7 & 69.5 & \textbf{99.7} & 0.57 & 0.58 & 5.9 \\

& CharRNN~\tabcite{segler2018generating}  
& 92.6 & 0.29 & \textbf{0.035}  
& 97.5 & 99.9 & 84.2 & 99.4 & \textbf{0.52} & 0.56 & 11.0 \\

& SAFE-GPT \tabcite{noutahi2024gotta} 
& 92.8 & 0.12 & 0.015  
& \textbf{99.8} & 98.9 & 43.7 & 97.7 & 0.72 & 0.57 & 6.3 \\

& GenMol \tabcite{lee2025genmol}  
& 62.1 & 0.05 & 0.012  
& 99.7 & 64.0 & 68.9 & 98.1 & 16.4 & \textbf{0.64} & 1.6 \\

\midrule
2D Graph 
& DiGress~\tabcite{vignac2022digress} 
& 82.5 & 0.26 & 0.031  
& 87.1 & \textbf{100.0} & 94.2 & 97.5 & 1.25 & 0.53 & 12.8 \\

& DisCo~\tabcite{xu2024discrete}  
& - & - & - 
& 88.3 & 100.0 & \textbf{\underline{97.7}} & 95.6 & 1.44 & 0.50 & 15.1 \\

& Cometh~\tabcite{siraudin2024cometh}  
& 82.1 & 0.36 & 0.023 
& 87.2 & 100.0 & 96.4 & 97.3 & 1.44 & 0.51 & \textbf{\underline{16.8}} \\

& DeFoG~\tabcite{qin2024defog}  
& 88.5 & 0.26 & 0.031  
& 92.8 & 99.9 & 92.1 & \underline{98.9} & 1.95 & 0.55 & 14.4 \\
\cmidrule{2-12}
& \textbf{MolHIT}  
& \textbf{\underline{94.2}} & \textbf{\underline{0.39}} & \underline{0.033} 
& \underline{99.1} & 99.8 & 91.4 & 98.0 & \underline{1.03} & \underline{0.55} & 14.4 \\
\bottomrule
\end{tabular}}
\vspace{-0.1in}
\end{table*}

\paragraph{Temperature sampling}

While temperature and top-$p$ sampling have become standard techniques for managing the quality-diversity trade-off in generative domains~\citep{holtzman2019curious, ficler2017controlling, hashimoto2019unifying}, their application to molecular graph generation remains largely unexplored. We evaluate the impact of these sampling strategies and demonstrate that our PN-sampler effectively controls this trade-off. We empirically find that temperature sampling can be naturally adopted for PN-sampler, where doing temperature sampling only for the atom prediction (line 7 in Alg.~\ref{alg:pn_sampler}) results in the best performance.

\subsection{Conditional Modeling}
\label{sec:method_conditional_modeling}
To enable conditional modeling, we train a conditional model by modifying the original graph transformer architecture in DiGress~\citep{vignac2022digress} by adding adaptive layer normalization (adaLN) for node attention only. For sampling, we adopt classifier-free guidance (CFG)~\citep{ho2022classifier}. We provide the details in Appendix~\ref{app:exp_details_cond_graph_transformer}.
\section{Experiments}
\label{sec:experiments}

We evaluate \textbf{MolHIT} on two large-scale molecular datasets: MOSES~\citep{polykovskiy2020molecular} and Guacamol~\citep{brown2019guacamol}. The MOSES dataset consists of 1.9M molecules containing 7 heavy atom types, which we augment into 12 tokens using DAE (Sec.~\ref{subsec:method_decoupled_atom_encoding}). Similarly, the GuacaMol~\citep{brown2019guacamol} dataset, which originally contains 12 heavy atom types, is decoupled into 56 tokens via DAE. For the model architecture, we utilize the original graph transformer from DiGress~\cite{vignac2022digress}, maintaining the same model size. All reported results represent the average of three independent runs, and standard deviations are provided in Appendix~\ref{app:full_stats_results}.

\subsection{Unconditional Generation on MOSES}

\paragraph{Evaluation} 
Following previous graph diffusion works~\citep{vignac2022digress, qin2024defog}, we measure with official benchmarks for Moses~\citep{polykovskiy2020molecular} which includes 7 metrics: Validity (\%), Uniqueness (\%), Novelty (\%), Filters (\%), FCD, SNN, Scaf. We also measure Quality~\cite{lee2025genmol}, which is defined by the proportion of molecules that are valid, unique, synthetic accessibility (SA~\citep{bickerton2012quantifying} $\leq$ 4), and drug-like (QED~\citep{ertl2009estimation} $\geq$ 0.6). Formal definitions of the metrics are provided in Appendix~\ref{app:moses_evaluation_metrics}.

\begin{table}[t]
\centering
\caption{Full GuacaMol benchmark results using unfiltered dataset. Metric abbreviations: Val. (Validity), V.U. (Unique), V.U.N. (Novel). DiGress (org.) is original DiGress trained with filtered dataset and DiGress (full) values are from the re-implementation of DiGress on unfiltered, full GuacaMol dataset.}
\label{tab:guacamol_main}
\vspace{-0.05in}
\resizebox{\columnwidth}{!}{ 
\renewcommand{\arraystretch}{0.98}
\renewcommand{\tabcolsep}{7pt}
\begin{tabular}{l ccccc}
\toprule
Model & Val. & V.U. & V.U.N. & KL Div. & FCD \\
\midrule
Training set & 100.0 & 100.0 & \textemdash & 99.9 & 92.8 \\
\midrule
DiGress~\tabcite{vignac2022digress} (org.) & 85.2 & 85.2 & 85.1 & 92.9 & \textbf{68.0} \\
DiGress~\tabcite{vignac2022digress} (full) 
& 74.7 & 74.6 & 74.0 & 92.4 & 61.1 \\
DiGress + DAE & 65.2 &  65.2 & 64.9 & 87.0 & 49.2 \\
\midrule
\textbf{MolHIT (Ours)} & \textbf{87.1} & \textbf{87.1} & \textbf{86.0} & \textbf{96.7} & 54.9 \\
\bottomrule
\end{tabular}
}
\vspace{-0.1in}
\end{table} 

\begin{table*}[t]
\centering
\caption{Multi-property guided generation on MOSES with four different conditions. We report mean absolute error (MAE), Pearson correlation (Pearson $r$), and validity. Avg. denotes the macro-average across four properties. Bold denotes best performances. All results are the averaged value over 3 runs of 10,000 samples.}
\label{tab:multi_property}
\vspace{-0.05in}
\resizebox{\textwidth}{!}{%
\renewcommand{\arraystretch}{0.98}
\renewcommand{\tabcolsep}{10pt}
\begin{tabular}{l
  c c c c c
  c c c c c
  c
}
\toprule
& \multicolumn{5}{c}{MAE $\downarrow$} & \multicolumn{5}{c}{Pearson $r \uparrow$} & {Validity (\%) $\uparrow$} \\
\cmidrule(r){2-6} \cmidrule(r){7-11} \cmidrule(l){12-12}
Method
& {QED} & {SA} & {logP} & {MW} & {\textbf{Avg.}}
& {QED} & {SA} & {logP} & {MW} &{\textbf{Avg.}}
& {} \\
\midrule
Marginal   
& 0.117 & 0.115 & 0.067 & 0.272 & 0.143
& 0.489 & 0.570 & 0.802 & 0.396 & 0.564
& 75.03 \\
Marginal + DAE  
& 0.107 & 0.094 & 0.061 & 0.227 & 0.122
& 0.565 & 0.559 & 0.836 & 0.437 & 0.599
& 87.85 \\
\midrule
\textbf{MolHIT (Ours)}  
& \bfseries 0.061 & \bfseries 0.040 & \bfseries 0.049 & \bfseries 0.081 & \bfseries 0.058
& \bfseries 0.804 & \bfseries 0.790 & \bfseries 0.950 & \bfseries 0.685 & \bfseries 0.807
& \bfseries 96.31 \\
\bottomrule
\end{tabular}
}
\vspace{-0.15in}
\end{table*}

\paragraph{Scaffold novelty metrics}
While the standard MOSES benchmark provides a foundation for evaluating molecular generative models, simple metrics like novelty may not reflect the capability for generating new molecules. For instance, a high novelty score itself can come from merely generating novel-looking noise outside the manifold of drug-like molecules, while high uniqueness may not reflect true structural diversity if the model is trapped in a narrow chemical subspace. 
To address this, we introduce two metrics given $n_{total}$ generated molecules: (1) Scaffold Novelty $= |\mathcal{S}_{\text{gen}} \setminus \mathcal{S}_{\text{train}}| / n_{\text{total}}$, which quantifies the efficiency of structural extrapolation; and (2) Scaffold Retrieval $= |\mathcal{S}_{\text{gen}} \cap \mathcal{S}_{\text{test}}| / n_{\text{total}}$, which measures distributional fidelity. Further details are in Appendix~\ref{app:structure_novelty_metric}.

\paragraph{Baselines} For graph generative models, we compare with DiGress~\citep{vignac2022digress}, DisCo~\citep{xu2024discrete}, Cometh~\citep{siraudin2024cometh}, DeFoG~\citep{qin2024defog}, which are previous SOTA in atom-level graph diffusion. We also compare with 1D baselines; VAE~\citep{kingma2013auto}, Char-RNN~\citep{segler2018generating}, SAFE-GPT~\citep{noutahi2024gotta}, GenMol~\citep{lee2025genmol}.

\paragraph{Result}
As shown in Table~\ref{tab:moses_comprehensive}, MolHIT significantly outperforms previous graph-based baselines across nearly all key metrics, including Quality, Validity, FCD, and Scaffold Novelty. While 1D sequence-based models (SAFE-GPT, GenMol) excel in Validity, they exhibit a clear tendency toward memorization, evidenced by their lower Scaf-Novelty and novelty scores. On the other hand, MolHIT achieves a new state-of-the-art both for Quality ($94.2\%$) and Scaffold Novelty ($0.39$) while achieving near perfect validity score ($99.1$). The above results validate that MolHIT effectively navigates the valid drug-like manifold without sacrificing its ability to explore novel chemical space.

\begin{table}[t!]
\centering
\caption{Scaffold extension results on the MOSES dataset. Results are averaged over 3 runs of 10,000 targets. 
}
\label{tab:scaffold_extension}
\vspace{-0.05in}
\resizebox{\linewidth}{!}{
\renewcommand{\arraystretch}{1.05}
\renewcommand{\tabcolsep}{3pt}
\begin{tabular}{l c c c c} 
\toprule
Model & Validity (\%) $\uparrow$  & Diversity $\uparrow$ & Hit@1 $\uparrow$ & Hit@5 $\uparrow$ \\
\midrule
DiGress & 50.8 & 44.8 & 2.07 & 6.41 \\
Marginal + DAE & 64.8 & \textbf{58.0} & 1.67 & 6.37 \\
\textbf{MolHIT (Ours)} & \textbf{83.9} & 57.4 & \textbf{3.92} & \textbf{9.79} \\
\bottomrule
\end{tabular}
}
\vspace{-0.1in}
\end{table}

\subsection{Unconditional Generation on GuacaMol}
\paragraph{Setup}
Compared to MOSES where molecules contain charged atoms are filtered, GuacaMol benchmark~\citep{brown2019guacamol} contains a broader chemical space, including compounds with formal charges that are not eliminated by neutralization. Previous atom encoding ~\citep{vignac2022digress} fails to reconstruct these properties (Fig.~\ref{fig:dae_reconstruction_comparison_guacamol}), and they train model only with a manually filtered dataset which are failed to be reconstructed. This helps improve the validity measure, but making models learn from the imperfect, biased distribution. In contrast, we utilize the full GuacaMol dataset for training to evaluate the robustness of our model. We run 3 run of generating 10,000 samples for each experiment.

\paragraph{Results}
Table~\ref{tab:guacamol_main} shows that MolHIT achieves the highest performance among all metrics except FCD. For FCD, the strong performance of original DiGress without DAE indicates that using DAE does not always lead to the generative task being easier since it can be hard to model with differentiate extended atom vocabulary. However, as in Appendix~\ref{app:DAE_in_guacamol}, we find that using DAE substantially increase the amount of molecules having charged or special atoms, which is not rare in the GuacaMol. Note that the original DiGress is trained for 1,000 epochs, while our results are from training only with 40 epochs, so further training will improve the metrics.

\subsection{Multi-property guided generation}
Generating molecules with targeted chemical properties is important for practical applications in materials science and drug discovery. For this, we evaluate the capacity of MolHIT under the multi-conditional generation scenario.

\paragraph{Setup}
We train a conditional graph transformer (Sec.~\ref{sec:method_conditional_modeling}) on the MOSES dataset, labeled with four key chemical properties: Quantitative Estimate of Drug-likeness (QED), Synthetic Accessibility (SA), Molecular Weight (MW), and the lipophilicity (logP). We utilize RDKit~\citep{rdkit}, an open-source cheminformatics toolkit, for all property labeling and condition evaluation.

\paragraph{Evaluation}
For inference, we generate 10,000 samples conditioned on target properties of the molecules that are randomly sampled from the test split. We measure Mean Absolute Error (MAE) and the Pearson correlation coefficient ($r$) for conditioning and validity for the structural fidelity of the samples. We compare MolHIT against two baselines: (1) Marginal transition (effectively a DiGress without a geometric prior) and (2) Marginal transition with a DAE (incorporating decoupled atom encoding into the marginal transition baseline).

\begin{table}[t]
\centering
\caption{Incremental performance gains on the MOSES dataset by integrating DAE, the PN Sampler, and HDDM into the DiGress.
}
\label{tab:ablation_component_anlaysis}
\vspace{-0.05in}
\resizebox{\linewidth}{!}{%
\renewcommand{\arraystretch}{0.98}
\renewcommand{\tabcolsep}{12pt}
\begin{tabular}{lccc}
\toprule
Method & Quality $\uparrow$ & FCD $\downarrow$ & Validity (\%) $\uparrow$ \\
\midrule
DiGress~\tabcite{vignac2022digress}  & 82.5 & 1.25 & 87.1 \\
+ DAE    & 87.6 & \textbf{0.89} & 96.2 \\
+ PN Sampler              & 92.9 & 1.65 & \textbf{99.4} \\
+ HDDM (\textbf{MolHIT})           & \textbf{94.2} & 1.03 & 99.1 \\
\bottomrule
\end{tabular}}
\vspace{-0.15in}
\end{table}

\paragraph{Results} Table~\ref{tab:multi_property} shows that MolHIT significantly outperforms all baselines across every metric. For conditioning precision, MolHIT achieves a macro-averaged MAE of 0.058, a 52.4\% reduction compared to the Marginal+DAE baseline. MolHIT also exhibits high reliability, reaching a Pearson $r$ of 0.807 on average, including a near-perfect 0.950 for $\log P$ and 0.804 for QED. The results also show this improved conditioning does not come with the cost of lower validity, where MolHIT achieves validity higher than 95\%, outperforming baselines with a large gap.
We provide more experimental details of the multi-property guided generation in Appendix~\ref{app:exp_details_multiprop}.

\begin{figure}[t] 
    \centering
    \includegraphics[width=\columnwidth]{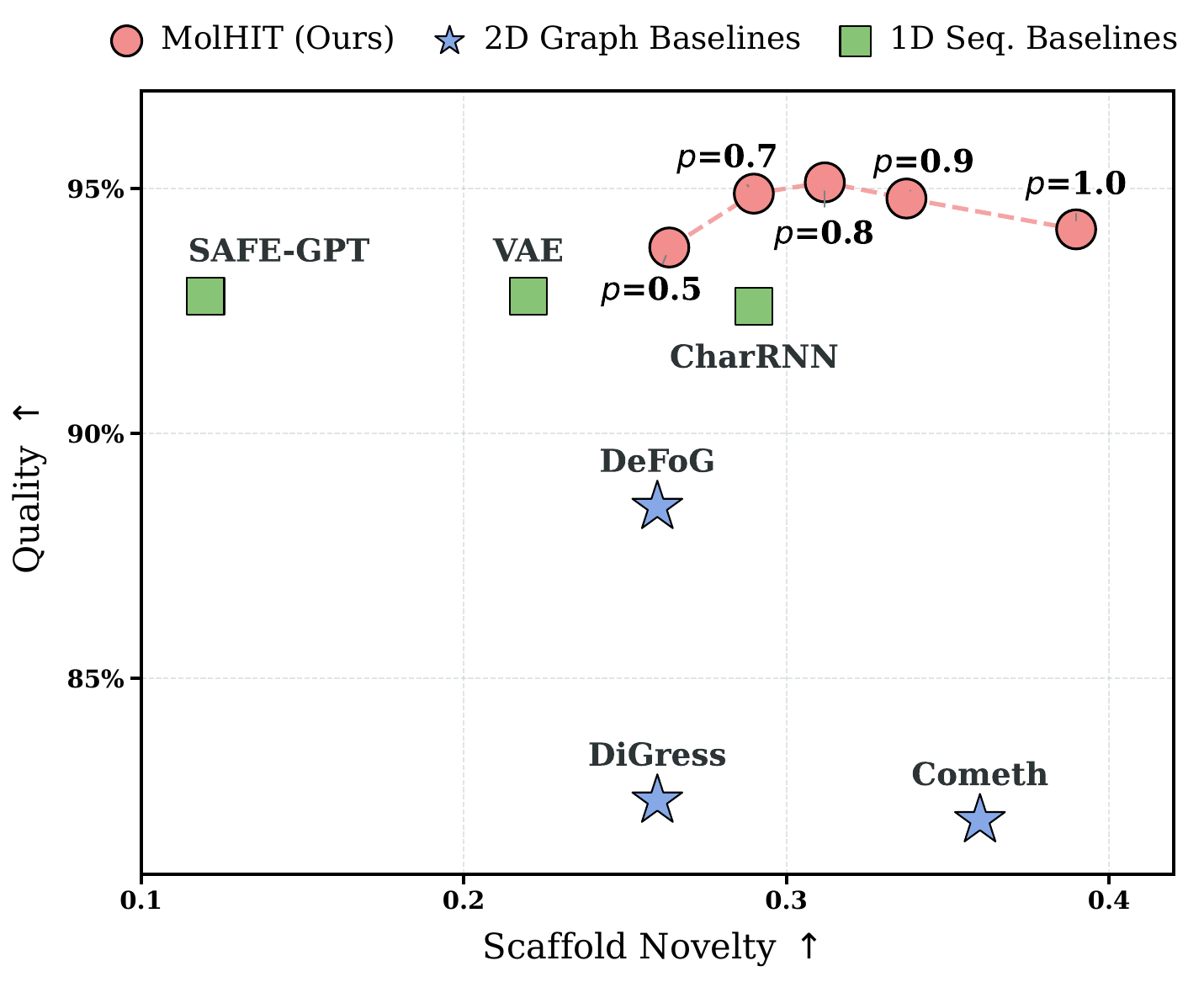}
    \vspace{-0.25in}
    \caption{Effect of top-p sampling in MolHIT.}
    \label{fig:top_p_ablation}
    \vspace{-0.1in}
\end{figure}

\subsection{Scaffold Extension}

\paragraph{Setup}
We evaluate pretrained unconditional model's generative capability when conditioned on a given substructure. For this, we use Bemis-Murcko scaffold~\citep{bemis1996properties} of the molecules in the test split and treat fix this part during the diffusion sampling. For each experiment, we utilize 10,000 unique target scaffolds, generating multiple candidates per target to assess the model's distributional coverage. Specifically, we measure the Hit@1 and Hit@5 ratios, which are the probability that the ground-truth extension is recovered within the top $k$ samples along with standard metrics of validity and diversity. Further experimental setup is provided in Appendix~\ref{app:exp_scaffold_extension}.

\paragraph{Result}
Table~\ref{tab:scaffold_extension} shows that MolHIT significantly outperforming DiGress in all metrics. Interestingly, applying DAE to DiGress improves validity and diversity while reducing in Hit@1, which may due to the extended expressivity of the model. However, DAE results in higher diversity, which results in matched Hit@5 ratio for original DiGress.

\subsection{Ablation Studies}

\paragraph{Component analysis}
To show the contribution of each component on MolHIT's performance, we conduct an ablation study by testing on the Moses dataset. The result in Table~\ref{tab:ablation_component_anlaysis} shows that our atom encoding method (DAE), sampler (PN sampler), and diffusion algorithm (HDDM) all contributes to get to the highest value of Quality, FCD, and Validity among graph diffusion models.

\paragraph{Effect of temperature sampling}
In Fig.~\ref{fig:top_p_ablation}, we analyze MolHIT trained on the MOSES dataset across a range of top-$p$ values. Our results demonstrate that as top-$p$ decreases, a clear trade-off emerges between sample quality and scaffold novelty. Specifically, lowering the top-$p$ value from 1.0 down to 0.8 consistently improves the quality and validity of generated structures, while further reducing the $p$ threshold leads to a sharp decline in both chemical metrics and structural diversity. Notably, when sampling with top-$p$, \textbf{MolHIT} achieves a high validity of 99.4\% and a quality score of 95.1\%, demonstrating the effectiveness of the nucleus sampling in \textbf{MolHIT}.

\section{Related Works}

\paragraph{Discrete diffusion models}
Along with the success of continuous diffusion models~\citep{ho2020denoising, song2020score}, discrete diffusion models formulate a noise process within a discrete state space. \citet{hoogeboom2021argmax} investigate uniform transition of the discrete diffusion models, while D3PM~\cite{austin2021structured} explore different types of transition mechanism which include absorbing transition. Recently and independently developed alongside our work, \citet{zhou2025next} propose a hierarchical discrete diffusion approach to language modeling. While similar in spirit, our HDDM is derived from a semigroup-consistent family of closed-form transition kernels $Q^{(1)}, Q^{(2)}$ parameterized by explicit diffusion scheduler $\alpha_t,\beta_t$ while \citet{zhou2025next} is developed in the CTMC framework~\citep{campbell2022continuous}. Moreover, HDDM supports an arbitrary row-stochastic projection $\mathbf{\Phi}$, which generalizes the deterministic hierarchical mapping used in \citet{zhou2025next}. 

\paragraph{Diffusion models for molecular generation}
Various diffusion models and its techniques have been applied for molecular graph generation. GDSS~\citep{jo2022score} formulate continuous diffusion modeling through the system of SDE with a score matching objective. DiGress~\citep{vignac2022digress} utilize primary form of discrete diffusion models with uniform-style transition with data dependent prior. \citet{siraudin2024cometh, xu2024discrete, qin2024defog} apply CTMC framework as in~\citet{campbell2022continuous} to further boost the performance. Another axis for molecular generation is to model 1D sequence. SAFE-GPT~\citep{noutahi2024gotta} trains an Autoregressive model with their unique representation of molecule while GenMol~\citep{lee2025genmol} adopts masked diffusion framework in a wide range of drug discovery tasks.
We defer further related works in Appendix~\ref{app:further_related_works}.
\label{sec:related_works}
\vspace{-10pt}

\section{Conclusion}
\label{sec:conclusion}
In this work, we present MolHIT, a novel molecular diffusion model with a hierarchical discrete diffusion framework. Our algorithm results in state-of-the-art performance in large molecular datasets. It unlocks new capacity for end-to-end atom-level molecular generation, directly generating atoms with formal charges or explicit nH for the first time, moving us towards more realistic molecule generation. 

\section*{Impact Statement}
This paper presents work whose goal is to advance the field of molecule generation. We hope our work accelerates the discovery of useful drugs and materials, improving human lives. However, one might maliciously use our model to generate harmful substances to humans and environments.

\bibliographystyle{plainnat}
\bibliography{reference}

\clearpage

\newpage
\appendix
\onecolumn

\let\addcontentsline\oldaddcontentsline
\renewcommand{\contentsname}{Table of contents}
\tableofcontents

\section{Limitation and future directions}
\label{app:Broader_impact_and_limitations}

\paragraph{Limitations}
While our models improve the traditional diffusion based molecular generation, we have not tested with the model size increase or architectural improvement, in which we believe have further room for better performance. Moreover, we have not fully trained the model until performance saturation on GuacaMol dataset and we believe performance improvement with further training.

\paragraph{Future directions}
There are many interesting future directions. One is to apply Hierarchical Discrete Diffusion Models to the language domains~\citep{sahoo2024simple}, or image models~\citep{chang2022maskgit} and combining with different sampling schemes for diffusion models~\citep{jung2024conditional, park2025temporal, wu2025fast, kim2025klass}. Another direction is to further improve MolHIT's framework with more advanced tokenization incorporating motifs or functional groups~\citep{jin2018junction} and apply into the 3D molecular generation~\citep{hoogeboom2022equivariant, xu2023geometric} and proteins~\citep{gruver2023protein}.

\section{Further background}
\label{app:Further_background}

\subsection{Further related works}
\label{app:further_related_works}

\paragraph{Further backgrounds on discrete diffusion models}
Recently, Masked Language Models are actively studied due to its simple form and potential of bi-directional modeling~\citep{devlin2019bert}. \citet{campbell2022continuous} formulate the discrete diffusion using Continuous Time Markov Chain (CTMC) framework and propose correction sampler leveraging tau-leaping. SEDD~\citep{lou2023discrete} introduce score entropy loss in analogous to the score matching loss in the continuous diffusion models and show scalability in language modeling. Recently, masked diffusion models is further simplified~\citep{sahoo2024simple,ou2024your,shi2024simplified} and reaching to the level that is comparable to the standard AR modeling in large scale~\citep{nie2024scaling, bie2025llada2} and even in multi-modal setting~\citep{yang2025mmada}.

\paragraph{Conditional generation with discrete diffusion}
For conditional generation, \citet{liu2024graph} propose graph diffusion transformer for multi-conditional generation on polymer dataset and shows the effectiveness of the classifier-free guidance. \citet{schiff2024simple} propose simple mechanism for conditional sampling which is in analogous with CFG in continuous diffusion.

\subsection{Details of masked diffusion models}

\paragraph{Marginal transition}
\label{app:further_backgrounds_marginal_transition}
Let $X=(X_1,\dots,X_N)$ be a discrete random vector with $X_k\in\{1,\dots,K\}$, and let $P$ denote the empirical data distribution on $\{1,\dots,K\}^N$.
Define the family of product of distributions:
\begin{equation}
\mathcal{C}
\;=\;
\left\{
q:\ q(x)=\prod_{k=1}^N q_k(x_k),\ \ q_k\in\Delta^{K-1}
\right\},
\end{equation}
where $\Delta^{K-1}$ is the $(K\!-\!1)$-simplex.

Then, one can define the marginal terminal distribution $\pi$ as the product of the data marginals:
\begin{equation}
\pi(x)
\;=\;
\prod_{k=1}^N \pi_k(x_k),
\qquad
\pi_k(a)
\;=\;
\mathbb{P}_{X\sim P}\!\left[X_k=a\right],
\ \ a\in\{1,\dots,K\}.
\end{equation}

Equivalently, $\pi$ is the (unique) KL projection of $P$ onto $\mathcal{C}$:
\begin{equation}
\pi
\;=\;
\arg\min_{q\in\mathcal{C}}
\mathrm{KL}\!\left(P\,\|\,q\right).
\end{equation}

Then, marginal transition defines forward process of discrete diffusion as follows:

\begin{equation}
    q(\mathbf{x}_t|\mathbf{x}_0) = \mathrm{Cat}(\mathbf{x}_t; \bar{\alpha}_t \mathbf{x}_0 + (1-\bar{\alpha}_t)\mathbf{\pi}).
\end{equation}

\section{Mathematical derivations}
\label{app:Mathematical_derivations}

\subsection{Generalized HDDM forward process}
\label{app:proof_of_closdeform_lemma}

We now derive a generalized forward process that incorporates arbitrary multi-level hierarchies. Let $\mathcal{T}$ be the total discrete state space with dimension $D = \sum_{k=0}^n K_k$. We partition $\mathcal{T}$ into $n+1$ disjoint subsets $\mathcal{S}_0, \mathcal{S}_1, \dots, \mathcal{S}_n$, where $\mathcal{S}_0$ represents the clean atomic states (with $|\mathcal{S}_0| = K_0$) and $\mathcal{S}_k$ represents the $k$-th level of intermediate hierarchical states (with $|\mathcal{S}_k| = K_k$). We further define the cumulative subspace up to level $i$ as $\mathcal{T}_i := \bigcup_{k=0}^i \mathcal{S}_k$. For each hierarchical stage $i \in \{1, \dots, n\}$, we define the transition kernel as a row-stochastic matrix $\mathbf{\Phi}_i \in [0, 1]^{|\mathcal{T}_{i-1}| \times K_i}$. This kernel encodes the probabilistic mapping from the cumulative lower-level states in $\mathcal{T}_{i-1}$ to the specific higher-level states in $\mathcal{S}_i$. To characterize the evolution in full space $\mathcal{T}$, we induce a global transition matrix $Q^{(i)} \in [0, 1]^{D \times D}$ on $\mathcal{T}$ which embeds the local kernel $\mathbf{\Phi}_i$ into the full space as follows:

\begin{equation}
Q^{(i)}(\mathbf{x}_{\text{next}} \mid \mathbf{x}) = 
\begin{cases} 
\mathbf{\Phi}_i(\mathbf{x}_{\text{next}} \mid \mathbf{x}) & \text{if } \mathbf{x} \in \mathcal{T}_{i-1} \text{ and } \mathbf{x}_{\text{next}} \in \mathcal{S}_i, \\ 
1 & \text{if } \mathbf{x} \in \mathcal{T} \setminus \mathcal{T}_{i-1} \text{ and } \mathbf{x}_{\text{next}} = \mathbf{x}, \\ 
0 & \text{otherwise}. 
\end{cases} 
\end{equation}

In matrix notation, $Q^{(i)}$ forms a block structure where the transitions from $\mathcal{T}_{i-1}$ are governed by $\mathbf{\Phi}_i$, while the remaining diagonal blocks form an identity matrix. Under this formulation, each $Q^{(i)}$ represents the probabilistic projection onto the $i$-th hierarchical level, enabling us to design the diffusion forward process with the following lemma:

\begin{proposition}
\label{proposition:generalized_forward_process}
Suppose monotonically decreasing functions $\alpha_t^{(i)}:=\alpha_i(t)$ ($i=1,2,...,n$) defined in $0\leq t\leq 1$ are satisfying $0\leq\alpha_t^{(1)}\leq\cdots\leq\alpha_t^{(n)}\leq1$ and boundary conditions $\alpha_0^{(i)}=1, \alpha_{1}^{(i)}=0$ for all $i$. We define the transition matrix from timestep $s$ to timestep $t$ ($s\leq t)$ as:
\begin{equation}
\label{eq:general_forward_kernel_t_from_s}
Q_{t|s} := \alpha_{t|s}^{(1)}I + (\alpha_{t|s}^{(2)}-\alpha_{t|s}^{(1)})Q^{(0)}+\cdots+(1-\alpha_{t|s}^{(n)})Q^{(n)},
\end{equation}
where $\alpha_{t|s}^{(i)}=\frac{\alpha_t^{(i)}}{\alpha_s^{(i)}}$ for every $i$. Then, transition kernel defined by Eq.~\ref{eq:general_forward_kernel_t_from_s} satisfies Chapman–Kolmogorov consistency~\citep{durrett2019probability} as follows:
\begin{equation}
\label{eq:Chapman-Komogorov_consistency}
    Q_{t|s}Q_{s|r} = Q_{t|r} \;\; \forall{0\leq r \leq s \leq t \leq 1.}
\end{equation}

Moreover, one could represent cumulative forward transition from initial timestep $0$ to $t$ in the following form:
\begin{equation}
    Q_t = \alpha_t^{(1)}\mathrm{I} + (\alpha_t^{(2)}-\alpha_t^{(1)})Q^{(1)} + \cdots (1-\alpha_t^{(n)})Q^{(n)}.
\end{equation}
\end{proposition}

\begin{proof}

First, we can observe $Q^{(i)}$ is a projection operator; i.e, ${Q^{(i)}}^2=Q^{(i)}$ for all $i$ by definition. In fact, this can be generalized as $Q^{(i)}Q^{(j)} = Q^{max(i,j)}$ for any $1\leq i,j \leq n$ by the definition of the $Q^{(i)}$ and $\varphi_i$. Now, suppose Eq.~\ref{eq:general_forward_kernel_t_from_s} holds for some $j\in \mathbb{N}$. Then, one can observe:
\begin{equation}
\begin{aligned}
     & Q_{t|s}Q_{s|r}\\
     & =  \left(
     \alpha_{t|s}^{(1)}I + (\alpha_{t|s}^{(2)}-\alpha_{t|s}^{(1)})Q^{(0)}+\cdots+(1-\alpha_{t|s}^{(j+1)})Q^{(j+1)}\right) 
     \left(
     \alpha_{s|r}^{(1)}I + (\alpha_{s|r}^{(2)}-\alpha_{s|r}^{(1)})Q^{(0)}+\cdots+(1-\alpha_{s|r}^{(j+1)})Q^{(j+1)}\right) \\
          & =  \left(
     \alpha_{t|s}^{(1)}I + (\alpha_{t|s}^{(2)}-\alpha_{t|s}^{(1)})Q^{(0)}+\cdots+(\alpha_{t|s}^{(j+1)}-\alpha_{t|s}^{(j)})Q^{(j)}\right) 
     \left(
     \alpha_{s|r}^{(1)}I + (\alpha_{s|r}^{(2)}-\alpha_{s|r}^{(1)})Q^{(0)}+\cdots+(\alpha_{s|r}^{(j+1)}-\alpha_{s|r}^{(j)})Q^{(j)}\right) \\
     &   + \left(\alpha_{t|s}^{(1)}+\cdots+(1-\alpha_{t|s}^{(j+1)})\right)\left(1-\alpha_{s|r}^{(j+1)}\right)Q^{(j+1)} + \left(1-\alpha_{t|s}^{(j+1)}\right)\left(\alpha_{s|r}+\cdots (1-\alpha_{s|r}^{(j+1)})\right)Q^{(j+1)} 
     \\ 
      & + \left(1-\alpha_{t|s}^{(j+1)})(1-\alpha_{s|r}^{(j+1)})\right)(Q^{(j+1)})^2 \\
      & =  \alpha_{t|s}^{(j+1)}\left(
     \frac{\alpha_{t|s}^{(1)}}{ \alpha_{t|s}^{(j+1)}}I +\cdots+ 
     \frac{(\alpha_{t|s}^{(j+1)}-\alpha_{t|s}^{(j)})}{\alpha_{t|s}^{(j+1)}}Q^{(j)}\right)\cdot \alpha_{s|r}^{(j+1)}
     \left(
     \frac{\alpha_{s|r}^{(1)}}{\alpha_{s|r}^{(j+1)}}I
     + \cdots+\frac{(\alpha_{s|r}^{(j+1)}-\alpha_{s|r}^{(j)})}{\alpha_{s|r}^{(j+1)}}Q^{(j)}\right) \\
     &  + \left(\alpha_{t|s}^{(1)}+\cdots+(1-\alpha_{t|s}^{(j+1)})\right)\left(1-\alpha_{s|r}^{(j+1)}\right)Q^{(j+1)} + \left(1-\alpha_{t|s}^{(j+1)}\right)\left(\alpha_{s|r}+\cdots (1-\alpha_{s|r}^{(j+1)})\right)Q^{(j+1)} 
     \\ 
      & + \left(1-\alpha_{t|s}^{(j+1)})(1-\alpha_{s|r}^{(j+1)})\right)(Q^{(j+1)})^2 \\
     & =  \left(
     \alpha_{t|r}^{(1)}I + (\alpha_{t|r}^{(2)}-\alpha_{t|r}^{(1)})Q^{(0)}+\cdots+(1-\alpha_{t|r}^{(j)})Q^{(j)}\right) + \left(\alpha_{t|s}^{(1)}+\cdots+(1-\alpha_{t|s}^{(j+1)})\right)\left(1-\alpha_{s|r}^{(j+1)}\right)Q^{(j+1)}  \\
     & + \left(1-\alpha_{t|s}^{(j+1)}\right)\left(\alpha_{s|r}+\cdots (1-\alpha_{s|r}^{(j+1)})\right)Q^{(j+1)}  + \left(1-\alpha_{t|s}^{(j+1)})(1-\alpha_{s|r}^{(j+1)})\right)(Q^{(j+1)})^2 \\
     & = \alpha_{t|r}^{(1)}I + (\alpha_{t|r}^{(2)} - \alpha_{t|r}^{(1)}) + \cdots + (1-\alpha_{t|r}^{(j+1)})Q^{(j+1)},
\end{aligned}
\end{equation}
where the second to the last equation comes from the inductive assumption on $j$. Since $j=1$ case is trivial, the result follows by mathematical induction on $j$.
\end{proof}

\paragraph{Proof of Lemma~\ref{lemma:HDDM_forward_closed_form}}
Lemma~\ref{lemma:HDDM_forward_closed_form} is now just a special case of above generalized formula in Proposition~\ref{proposition:generalized_forward_process}.

\subsection{Proof of Theorem~\ref{theorem: continuous time NELBO}}
\label{app:proof_NELBO}

Let $\mathbf{m}$ be the one-hot representation of the masked state which we take as a prior, and let $\mathbf{x} \in S_0$ denote an element in a clean state. For a forward transition kernel $Q_{t|s}$ defined induced by the row stochastic matrix $\mathbf{\Phi} \in [0, 1]^{K \times G}$ and the masking operation as in  Lemma~\ref{lemma:HDDM_forward_closed_form}, the conditional transition is defined as $q(\mathbf{z}_t|\mathbf{z}_s) = \mathrm{Cat}(\mathbf{z}_t; \mathbf{z}_s Q_{t|s})$. By applying the closed-form transition from Lemma~\ref{lemma:HDDM_forward_closed_form} and Bayes' rule, the posterior distribution $q(\mathbf{z}_s | \mathbf{z}_t, \mathbf{x})$ can be derived as follows:

\begin{equation}
\begin{aligned}
    & q(\mathbf{z}_s | \mathbf{z}_t,\mathbf{x}) = \mathrm{Cat}\left(\mathbf{z}_s ; \frac{\left[\alpha_{t|s}\mathrm{I} + (\beta_{t|s}-\alpha_{t|s})Q^{(1)}+(1-\beta_{t|s})Q^{(2)}\right]^{\intercal}\mathbf{z}_t\;\odot\;\left[\alpha_s\mathbf{x}+(\beta_s - \alpha_s)Q^{(1)}\mathbf{x}+(1-\beta_s)\mathbf{m}\right]}{\mathbf{z}_t^T[\alpha_{t}\mathrm{I}+(\beta_t-\alpha_t)Q^{(1)}+(1-\beta_t)Q^{(2)}]^T\mathbf{x}}\right)
\end{aligned}
\end{equation}

We can divide into the following 3 cases depending on which state $\mathbf{z}_t$ belongs to.
\paragraph{Case 1. $\mathbf{z}_t\in S_0$}
\begin{equation}
    q(\mathbf{z}_s | \mathbf{z}_t,\mathbf{x}) = \mathrm{Cat}(\mathbf{z}_s ; \mathbf{x}).
\end{equation}

\paragraph{Case 2. $\mathbf{z}_t\in S_1$}
When the observed state at time $t$ belongs to the mid-level space $S_1$, the posterior depends on the state of $\mathbf{z}_t$ under the stochastic transition. Using the block structure of the transition kernels, we can obtain:
\begin{equation}
\begin{aligned}
    q(\mathbf{z}_s | \mathbf{z}_t, \mathbf{x})
    &= \mathrm{Cat}\left(\mathbf{z}_s ; \frac{(\alpha_s \beta_{t|s} - \alpha_t) \mathbf{x} + (\beta_t - \beta_{t|s} \alpha_s) \mathbf{z}_t}{\beta_t - \alpha_t}\right),
\end{aligned}
\end{equation}

\paragraph{Case 3. $\mathbf{z}_t\in S_2=\{\mathbf{m}\}$}

When the observed state is the mask token $\mathbf{m} \in S_2$, the posterior distribution $q(\mathbf{z}_s | \mathbf{z}_t, \mathbf{x})$ becomes a weighted combination of the clean state, its stochastic projection, and the mask prior. Using the normalization constant $1 - \beta_t$, the posterior is given by:
\begin{equation}
    q(\mathbf{z}_s | \mathbf{z}_t, \mathbf{x}) = \mathrm{Cat}\left(\mathbf{z}_s ; \frac{\alpha_s(1 - \beta_{t|s})\mathbf{x} + (1 - \beta_{t|s})(\beta_s - \alpha_s)Q^{(1)}\mathbf{x} + (1 - \beta_s)\mathbf{m}}{1 - \beta_t}\right).
\end{equation}

\paragraph{Parameterization}

Inspired by the masked diffusion literature~\citep{sahoo2024simple, shi2024simplified, ou2024your}, we derive simplified loss form through the parameterizing a neural network $\theta$ to estimate only the probability in clean final state in $\mathcal{S}_0$. This leads to the posterior $p_{\theta}(\mathbf{z}_s|\mathbf{z}_t)$ in following closed forms depending on the current state $\mathbf{z}_t$.

\paragraph{Case 1. $\mathbf{z}_t\in S_0$}
\begin{equation}
    p_{\theta}(\mathbf{z}_s|\mathbf{z}_t) = \mathrm{Cat}(\mathbf{z_s;\mathbf{z}_t}).
\end{equation}

\paragraph{Case 2. $\mathbf{z}_t\in S_1$}
\begin{equation}
    p_{\theta}(\mathbf{z}_s|\mathbf{z}_t) = \mathrm{Cat}\left(\mathbf{z_s};\frac{(\alpha_s\beta_{t|s}-\alpha_t){Q^{(1)}}^{\intercal}\mathbf{z}_t\odot\mathbf{x}_{\theta}(\mathbf{z}_t,t)+(\beta_t-\beta_{t|s}\alpha_s)\mathbf{z}_t\odot Q^{(1)}\mathbf{x}_{\theta}}{(\beta_t-\alpha_t)\mathbf{z}_t^TQ^{(1)}\mathbf{x}_{\theta}}\right).
\end{equation}

\paragraph{Case 3. $\mathbf{z}_t\in S_2=\{\mathbf{m}\}$}
\begin{equation}
\begin{aligned}
    & p_{\theta}(\mathbf{z}_s|\mathbf{z}_t)  \\
    & =  \mathrm{Cat}\left(\mathbf{z_s};\frac{[\alpha_{t|s}\mathbf{m}+(\beta_{t|s}-\alpha_{t|s}){Q^{(1)}}^{\intercal}\mathbf{m}+(1-\beta_{t|s}){Q^{(2)}}^{\intercal}\mathbf{\mathbf{m}}]\odot[\alpha_s\mathbf{x}_{\theta}+(\beta_s-\alpha_s)Q^{(1)}\mathbf{x}_{\theta}+(1-\beta_s)\mathbf{m}]}{1-\beta_t}\right). \\
    & =  \mathrm{Cat}\left(\mathbf{z_s};\frac{\alpha_s(1-\beta_{t|s})\mathbf{x}_{\theta}+(1-\beta_{t|s})(\beta_s-\alpha_s)Q^{(1)}\mathbf{x}_{\theta}+(1-\beta_s)\mathbf{m}}{1-\beta_t}\right). 
\end{aligned}
\end{equation}

\paragraph{ELBO analysis}
Now, we start with analyzing the ELBO of the Hierarchical Discrete diffusion models in discrete timesteps.

\begin{equation}
    \label{eq:ELBO_discrete}
    L_T = \mathbb{E}_{t \in \{\frac{1}{T}, \frac{2}{T}, \dots, 1\}} \mathbb{E}_{q(\mathbf{z}_t|\mathbf{x})} \left[ T \cdot \mathrm{D_{KL}}\left( q(\mathbf{z}_s|\mathbf{z}_t, \mathbf{x}) \,\|\, p_\theta(\mathbf{z}_s|\mathbf{z}_t) \right) \right]
\end{equation}

\paragraph{Case 1. $\mathbf{z}_t\in S_0$}
\begin{equation}
    \mathrm{D_{KL}}\left( q(\mathbf{z}_s|\mathbf{z}_t, \mathbf{x}) \,\|\, p_\theta(\mathbf{z}_s|\mathbf{z}_t) \right) = 0.
\end{equation}

\paragraph{Case 2. $\mathbf{z}_t\in S_1$}
\begin{equation}
\begin{aligned}
    & \mathrm{D_{KL}}\left( q(\mathbf{z}_s|\mathbf{z}_t, \mathbf{x}) \,\|\, p_\theta(\mathbf{z}_s|\mathbf{z}_t) \right) \\
    = & \sum_{\mathbf{z}_s\in\{{\mathbf{x},\varphi(\mathbf{x})}\}}  q(\mathbf{z}_s|\mathbf{z}_t, \mathbf{x})\log\frac{q(\mathbf{z}_s|\mathbf{z}_t, \mathbf{x})}{p_{\theta}(\mathbf{z}_s|\mathbf{z}_t)} \\
    = & \; q(\mathbf{z}_s=\mathbf{x}|\mathbf{z}_t, \mathbf{x})\log\frac{q(\mathbf{z}_s=\mathbf{x}|\mathbf{z}_t, \mathbf{x})}{p_{\theta}(\mathbf{z}_s=\mathbf{x}|\mathbf{z}_t)} + q(\mathbf{z}_s=\varphi(\mathbf{x})|\mathbf{z}_t, \mathbf{x})\log\frac{q(\mathbf{z}_s=Q^{(1)}\mathbf{x}|\mathbf{z}_t, \mathbf{x})}{p_{\theta}(\mathbf{z}_s=Q^{(1)}\mathbf{x}|\mathbf{z}_t)} \\
    = & \; \frac{(\alpha_s\beta_{t|s}-\alpha_t)}{\beta_t-\alpha_t}\log\frac{\mathbf{z}_t^{\intercal}Q^{(1)}\mathbf{x}_{\theta}}{\langle\mathbf{x},{Q^{(1)}}^{\intercal}\mathbf{z}_t\rangle\cdot\langle\mathbf{x}_{\theta},\mathbf{x}\rangle} + 0.
\end{aligned}
\end{equation}

\paragraph{Case 3. $\mathbf{z}_t\in S_2=\{\mathbf{m}\}$}
\begin{equation}
\begin{aligned}
    & \mathrm{D_{KL}}\left( q(\mathbf{z}_s|\mathbf{z}_t, \mathbf{x}) \,\|\, p_\theta(\mathbf{z}_s|\mathbf{z}_t) \right) \\
    = & \sum_{\mathbf{z}_s\in\{{\mathbf{x},\varphi(\mathbf{x}),\mathbf{m}}\}}  q(\mathbf{z}_s|\mathbf{z}_t, \mathbf{x})\log\frac{q(\mathbf{z}_s|\mathbf{z}_t, \mathbf{x})}{p_{\theta}(\mathbf{z}_s|\mathbf{z}_t)} \\
    = & \; q(\mathbf{z}_s=\mathbf{x}|\mathbf{z}_t=\mathbf{m}, \mathbf{x})\log\frac{q(\mathbf{z}_s=\mathbf{x}|\mathbf{z}_t=\mathbf{m}, \mathbf{x})}{p_{\theta}(\mathbf{z}_s|\mathbf{z}_t)} + q(\mathbf{z}_s=Q^{(1)}\mathbf{x}|\mathbf{z}_t=\mathbf{m}, \mathbf{x})\log\frac{q(\mathbf{z}_s=Q^{(1)}\mathbf{x}|\mathbf{z}_t=\mathbf{m}, \mathbf{x})}{p_{\theta}(\mathbf{z}_s|\mathbf{z}_t)} \\
    + & \; q(\mathbf{z}_s=\mathbf{m}|\mathbf{z}_t=\mathbf{m}, \mathbf{x})\log\frac{q(\mathbf{z}_s=\mathbf{m}|\mathbf{z}_t=\mathbf{m}, \mathbf{x})}{p_{\theta}(\mathbf{z}_s|\mathbf{z}_t)} \\
    = & \frac{(\alpha_s-\alpha_s\beta_{t|s})}{1-\beta_t}\log\frac{1}{\langle\mathbf{x}_{\theta},\mathbf{x}\rangle} + \frac{(1-\beta_{t|s})(\beta_s-\alpha_s)}{1-\beta_t}\KL\left(Q^{(1)}\mathbf{x}\|Q^{(1)}\mathbf{x}_{\theta}(\mathbf{z}_t,t)\right) + 0. 
\end{aligned}
\end{equation}

Combined together, each term in Eq.~\ref{eq:ELBO_discrete} can be expressed as follows:

\begin{equation}
\begin{aligned}
    & \mathrm{D_{KL}}\left( q(\mathbf{z}_s|\mathbf{z}_t, \mathbf{x}) \,\|\, p_\theta(\mathbf{z}_s|\mathbf{z}_t) \right)  \\
    & =
    \frac{(\alpha_s\beta_{t|s}-\alpha_t)}{\beta_t-\alpha_t}\left(\log\langle{\mathbf{z}_t,Q^{(1)}\mathbf{x}_{\theta}}\rangle-\log{\langle\mathbf{x}_{\theta},\mathbf{x}\rangle}\right)\langle\mathbf{z}_t,Q^{(1)}\mathbf{x}\rangle \\
    & - \left( \frac{(1-\beta_{t|s})(\beta_s-\alpha_s)}{1-\beta_t}\log\langle Q^{(1)}\mathbf{x},Q^{(1)}\mathbf{x}_{\theta}(\mathbf{z}_t,t) \rangle +\frac{(\alpha_s-\alpha_s\beta_{t|s})}{1-\beta_t}\log{\langle\mathbf{x}_{\theta},\mathbf{x}\rangle}\right)\langle\mathbf{z}_t,\mathbf{m}\rangle + C_{s,t},
\end{aligned}
\end{equation}
for some constant $C_{s,t}$, which leads into following continuous time NELBO of HDDM:

\begin{equation}
\begin{aligned}
    & \mathcal{L}_{\text{NELBO}}^{\infty}(\theta)  
    \\ & = 
    \lim_{T\rightarrow \infty} \sum_{i=2}^{T}\mathrm{D_{KL}}\left( q(\mathbf{z}_{t(i-1)}|\mathbf{z}_{t(i)}, \mathbf{x}) \,\|\, p_\theta(\mathbf{z}_{t(i-1)}|\mathbf{z}_{t(i)}) \right) 
    \\
    & =  \lim_{T\rightarrow \infty} \sum_{i=2}^{T}
    \frac{(\alpha_{t(i)}\beta_{t(i)|t(i-1)}-\alpha_{t(i)})}{\beta_{t(i)}-\alpha_{t(i)}}\left(\log\langle{\mathbf{z}_{t(i)},Q^{(1)}\mathbf{x}_{\theta}}\rangle-\log{\langle\mathbf{x}_{\theta},\mathbf{x}\rangle}\right)\langle\mathbf{z}_{t(i)},Q^{(1)}\mathbf{x}\rangle \\
    & - \left( \frac{(1-\beta_{t(i)|t(i-1)})(\beta_{t(i)}-\alpha_{t(i-1)})}{1-\beta_{t(i)}}\log\langle Q^{(1)}\mathbf{x},Q^{(1)}\mathbf{x}_{\theta}(\mathbf{z}_{t(i)},t(i)) \rangle +\frac{(\alpha_{t(i-1)}-\alpha_{t(i-1)}\beta_{t(i)|t(i-1)})}{1-\beta_{t(i)}}\log{\langle\mathbf{x}_{\theta},\mathbf{x}\rangle}\right)\langle\mathbf{z}_{t(i)},\mathbf{m}\rangle 
    \\ & + C_{t(i-1),t(i)}
    \\ & = 
    \int_{t(1)}^{1} \left[\frac{\alpha_t(\beta_{t}'-\alpha_t')}{\beta_t-\alpha_t}\left(\log\langle{\mathbf{z}_t,Q^{(1)}\mathbf{x}_{\theta}}\rangle-\log{\langle\mathbf{x}_{\theta},\mathbf{x}\rangle}\right)\langle\mathbf{z}_t,Q^{(1)}\mathbf{x}\rangle \right. \\
    & \left. - \left( \frac{\beta_t'(\beta_t - \alpha_t)}{\beta_t(1-\beta_t)}\log\langle Q^{(1)}\mathbf{x},Q^{(1)}\mathbf{x}_{\theta}(\mathbf{z}_t,t) \rangle +\frac{\alpha_t\beta_t'}{\beta_t(1-\beta_t)}\log{\langle\mathbf{x}_{\theta},\mathbf{x}\rangle}\right)\langle\mathbf{z}_t,\mathbf{m}\rangle + C_t\right] dt.
\end{aligned}
\end{equation}

As a result, we obtain the general NELBO of HDDM in the following Theorem:

\begin{tcolorbox}[
    colback=gray!15,             
    colframe=black!60,           
    title=General NELBO in HDDM  
]
\begin{theorem}
\label{theorem:general_HDDM_NELBO}
Suppose forward process of two-level HDDM $Q$ is defined with stochastic operator $Q^{(1)}, Q^{(2)}$, which are induced from the $\mathbf{\Phi}$ and masking operator, respectively as in Lemma~\ref{lemma:HDDM_forward_closed_form}. Then, for some constant $C_1$, the negative evidence lower bound (NELBO) can be expressed as follows:
\begin{equation}
\begin{aligned}
& \mathcal{L}_{\text{NELBO}}^{\infty}(\theta)  = \mathbb{E}_{Q,t} \frac{\alpha_t(\beta_{t}'-\alpha_t')}{\beta_t-\alpha_t}\left(\log\langle{\mathbf{z}_t,Q^{(1)}\mathbf{x}_{\theta}}\rangle-\log{\langle\mathbf{x}_{\theta},\mathbf{x}\rangle}\right)\langle\mathbf{z}_t,Q^{(1)}\mathbf{x}\rangle \\
    & - \left( \frac{\beta_t'(\beta_t - \alpha_t)}{\beta_t(1-\beta_t)}\log\langle Q^{(1)}\mathbf{x},Q^{(1)}\mathbf{x}_{\theta}(\mathbf{z}_t,t) \rangle +\frac{\alpha_t\beta_t'}{\beta_t(1-\beta_t)}\log{\langle\mathbf{x}_{\theta},\mathbf{x}\rangle}\right)\langle\mathbf{z}_t,\mathbf{m}\rangle  + C_1.
\end{aligned}
\end{equation}
\end{theorem}
\end{tcolorbox}

Note that above theorem holds for arbitrary stochastic row matrix $\mathbf{\Phi}$, which means we can design any stochastic mapping from $\mathcal{S}_0$ to $\mathcal{S}_1$. 

When $Q^{(1)}$ is deterministic (i.e, its rows are composed of one-hot vectors),  we can parameterize model to estimate the categories that are in the same mid-level state. This parameterization leads to further simplified form of Theorem~\ref{theorem:general_HDDM_NELBO} as follows:

\begin{tcolorbox}[
    colback=gray!15,             
    colframe=black!60,           
    title=NELBO in HDDM with deterministic grouping 
]
\begin{corollary}
\label{theorem:HDDM_NELBO_for_deterministic_grouping}
\begin{equation}
\begin{aligned}
& \mathcal{L}_{\text{NELBO}}^{\infty}(\theta)  = \mathbb{E}_{Q,t}\frac{\alpha_t(\beta_{t}'-\alpha_t')}{\beta_t-\alpha_t}\log{\langle\mathbf{x}_{\theta}(\mathbf{z}_t,t),\mathbf{x}\rangle}\cdot\mathbb{I}\left[\mathbf{z}_t\in \mathcal{S}_1\right] \\ & 
- \frac{\beta_{t}'}{\beta_t(1-\beta_t)}\left((\beta_t-\alpha_t)\log\langle Q^{(1)}\mathbf{x},Q^{(1)}\mathbf{x}_{\theta}(\mathbf{z}_t,t) \rangle +\alpha_t\log{\langle\mathbf{x}_{\theta}(\mathbf{z}_t,t),\mathbf{x}\rangle}\right)\mathbb{I}\left[\mathbf{z}_t=\mathbf{m}\right] + C_2,
\end{aligned}
\end{equation}
for some constant $C_2$.
\end{corollary}
\end{tcolorbox}

\section{Experiment details}
\label{app:Experiment_details}

\subsection{Decoupled Atom Encoding (DAE)}
\label{app:DAE}
While standard graph-based diffusion models typically adopt a coarse node encoding based solely on atomic numbers ($Z$), decoupled atom encoding (DAE) expands the original token vocabulary by explicitly decoupling three standards: aromaticity, hydrogen saturation, and formal charge magnitude. Decoupled Atom Encoding (DAE) expands the token vocabulary by explicitly decoupling three critical chemical descriptors: aromaticity, hydrogen saturation, and formal charge magnitude. Unlike previous methods that rely on implicit hydrogen estimation (e.g., via RDKit’s valence rules), DAE treats these attributes as primary node features to be explicitly encoded and decoded. This approach resolves the one-to-many mapping problem between atomic tokens and their physical states, enabling the near-perfect reconstruction of drug-like scaffolds from the MOSES and GuacaMol datasets. Furthermore, this extended vocabulary facilitates the reliable generation of complex heteroaromatics and zwitterionic species which are extremely rare for baselines using coarse tokenization. Specifically, we emphasize that tokenizing $[nH]$ as a distinct state is fundamentally different from modeling explicit hydrogen atoms as separate nodes. While the latter can significantly increases graph complexity and computational overhead, DAE preserves graph sparsity while maintaining chemical precision.

\begin{table}[ht]
\centering
\caption{Comparison of Atom Vocabularies on the MOSES Dataset. DAE resolves structural ambiguities by decoupling elements into specific aromatic and hydrogen-locked states.}
\label{tab:DAE_moses_vocabs}
\begin{small}
\begin{tabular}{lllc}
\toprule
\textbf{Method} & \textbf{Elemental Basis} & \textbf{Unique Tokens (Vocabulary)} & \textbf{Size} \\ 
\midrule
Standard Encoding & \{C, N, S, O, F, Cl, Br\} & C, N, S, O, F, Cl, Br & 7 \\
\midrule
\textbf{DAE (Ours)} & \{C, N, S, O, F, Cl, Br\} & \textbf{Aliphatic:} C, N, S, O, F, Cl, Br & 12 \\
& & \textbf{Aromatic:} c, n, \underline{nH}, s, o & \\
\bottomrule
\end{tabular}
\end{small}
\end{table}

\paragraph{DAE in MOSES}
The MOSES dataset consists primarily of stable, neutral drug-like molecules which is clean lead filtered from the ZINC dataset~\citep{irwin2005zinc}. In this context, the reconstruction bottleneck is primarily structural. Previous coarse-grained encodings fail to resolve the placement of pyrrolic hydrogens ($[nH]$), a critical motif in heteroaromatic rings like indole or imidazole. By explicitly decoupling aromaticity and hydrogen counts, DAE enables model can explicitly distinguish these motifs, resulting in improved generation quality.

\begin{table}[ht]
\centering
\caption{Vocabulary expansion for the Guacamol dataset. DAE scales from 12 elemental types to 56 semantic tokens including aromatic and charged atoms.}
\label{tab:guacamol_vocabulary}
\begin{small}
\begin{tabular}{lll}
\toprule
\textbf{Category} & \textbf{Standard Encoding (Size: 12)} & \textbf{DAE Tokens (Size: 56)} \\ 
\midrule
Neutral Aliphatic & \{C, N, O, F, B, Br, Cl, I, P, S, Se, Si\} & C, N, O, F, B, Br, Cl, I, P, S, Se, Si \\
\midrule
Aromatic States & (None / Implicit) & c, c+, c-, n, nH, n+, nH+, n-, s, s+, o, o+, se, se+, p \\
\midrule
Charged \& & (None / Implicit) & C+, C-, N+, NH+, NH2+, NH3+, N-, NH-, O+, O-, F+, F-, \\
Hypervalent & & B-, Br+2, Br-, Cl+, Cl+2, Cl+3, Cl-, I+, I+2, I+3, \\
& & P+, P-, S+, S-, Se+, Se-, Si- \\
\bottomrule
\end{tabular}
\end{small}
\end{table}

\paragraph{DAE in GuacaMol}
\label{app:DAE_in_guacamol}
GuacaMol~\citep{brown2019guacamol} is constructed from a standardized subset of ChEMBL~\citep{mendez2019chembl}, restricted to common medicinal-chemistry elements. In this unconstrained space, previous models suffer from a fundamental reconstruction failure; for instance, standard coarse-grained encoding achieves only a 1.88\% success rate on the [nH] group, with a negligible 0.09\% identity preservation rate. While previous models implicitly rely on the relaxation technique which can improve the success rates (e.g., increasing charged group success from 80.43\% to 96.54\%), this it only preserves 80.07\% of total molecules, indicating a failure to maintain the original chemical identity. As illustrated in Figure~\ref{fig:dae_reconstruction_comparison_guacamol}, \textbf{MolHIT} addresses this through Decoupled Atom Encoding (DAE), which expands the vocabulary to 56 tokens by encode-decode the atoms with extended vocabulary space, resulting in 100 \% success rate and over 99.98\% in identity preservation rate. Moreover, as illustrated in Fig.~\ref{fig:guacamol_formal_charge_ratio}, the effect of DAE also happens in generative performance, where it enables generating molecules with formal charge which consist of about $6\%$ in GuacaMol dataset. 

\begin{figure}[t]
    \centering
    \includegraphics[width=0.85\textwidth]{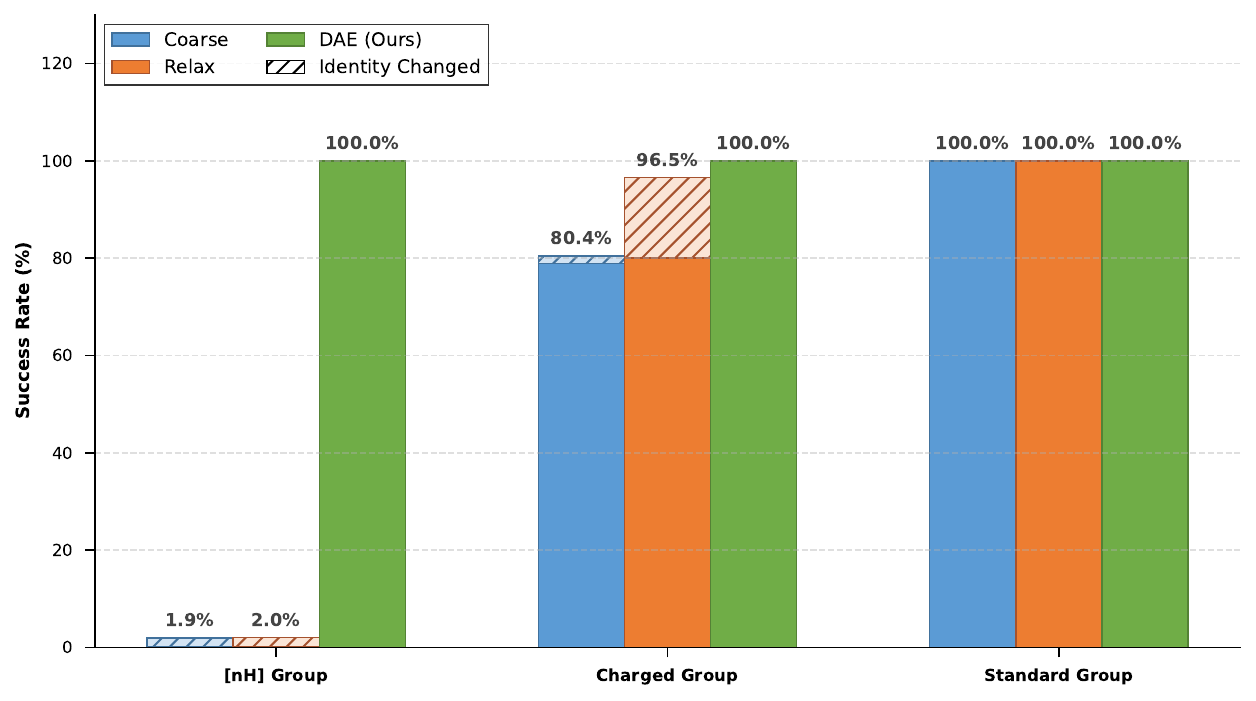} 
    \caption{The ratios of generated molecules having formal charge. \textbf{MolHIT} can reach to the training level proportion, while models with previous coarse encoding (left two) barely generate the charged atoms.}
    \label{fig:dae_reconstruction_comparison_guacamol}
\end{figure}

\begin{figure}[t]
    \centering
    \includegraphics[width=0.85\textwidth]{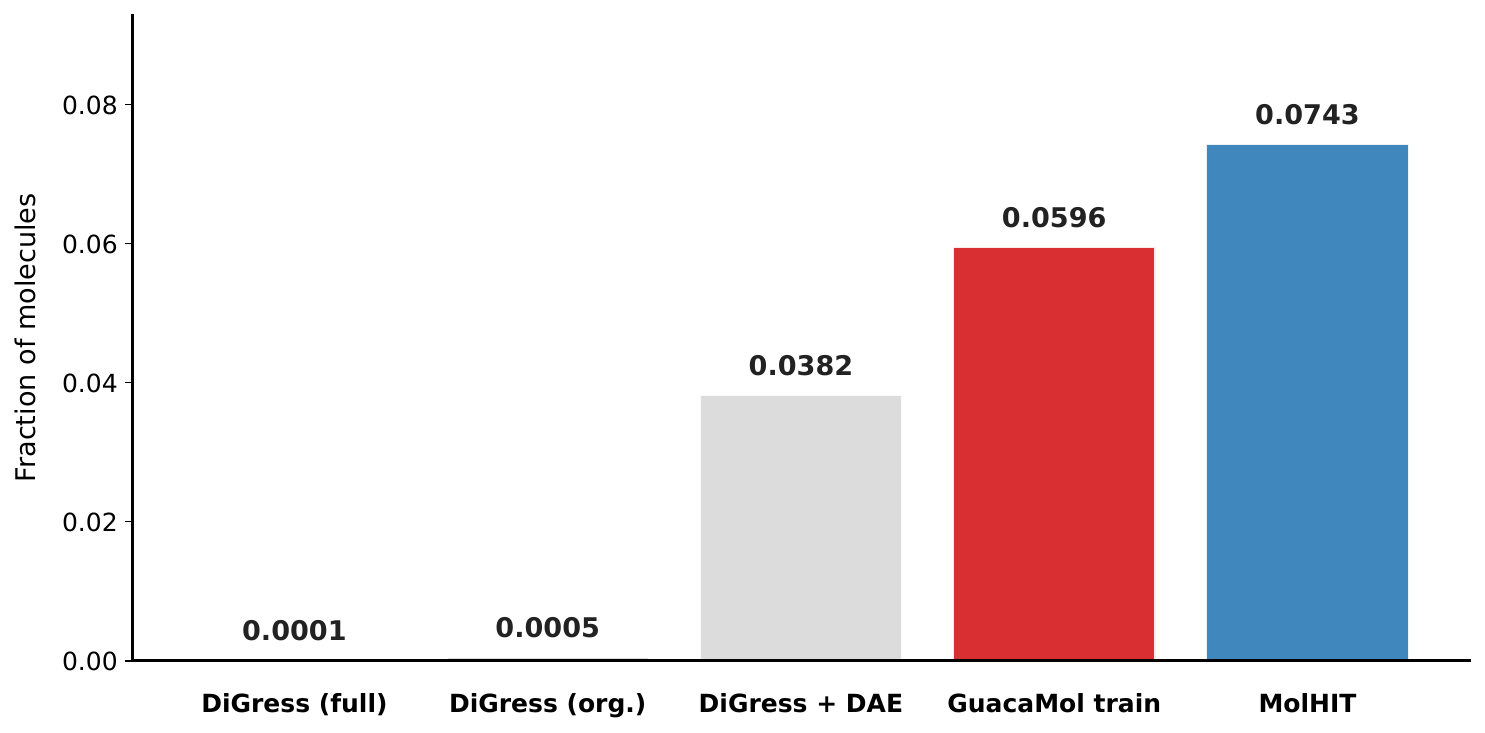} 
    \caption{\textbf{Reconstruction Fidelity and Identity Preservation.} We measure the proportion of generated molecules that have at least one atom with formal charge.}
    \label{fig:guacamol_formal_charge_ratio}
\end{figure}

\subsection{Grouping in HDDM}
\label{app:grouping_HDDM_details}

\paragraph{Grouping Details for MOSES and GuacaMol} \label{app:grouping_details}

We employ dataset-specific grouping strategies to align the intermediate state space $\mathcal{S}_1$ with the underlying chemical distribution of each corpus. Table \ref{tab:grouping_appendix} summarizes these partitions.

\begin{table}[ht]
\centering
\caption{Deterministic grouping kernels for node state space partitioning in MOSES and GuacaMol.}
\label{tab:grouping_appendix}
\begin{small}
\begin{tabular}{lll}
\toprule
\textbf{Dataset} & \textbf{Group ID} & \textbf{Atom Elements ($\mathcal{S}_0$)} \\
\midrule
\multirow{4}{*}{MOSES}    & Group 1 & \{C\} \\
                          & Group 2 & \{N, O, S\} \\
                          & Group 3 & \{F, Cl, Br\} \\
                          & Group 4 & \{c, o, n, [nH], s\} \\
\midrule
\multirow{6}{*}{GuacaMol} & Group 1 & \{F, Cl, Br, I, F$^-$, Cl$^-$, Br$^-$\} \\
                          & Group 2 & \{C, N, O, P, S, Se\} \\
                          & Group 3 & \{c, n, [nH], o, s, se, p\} \\
                          & Group 4 & \{N$^+$, n$^+$, [nH]$^+$, P$^+$, [NH]$^+$, [NH$_2$]$^+$, [NH$_3$]$^+$, Br$^{+2}$, Cl$^{+2}$, Cl$^{+3}$, I$^{+2}$, I$^{+3}$\} \\
                          & Group 5 & \{O$^-$, N$^-$, [NH]$^-$, O$^+$, S$^+$, B$^-$, C$^+$, C$^-$, c$^+$, c$^-$, n$^-$, s$^+$, o$^+$, se$^+$, F$^+$, Cl$^+$, I$^+$, P$^-$, S$^-$, Se$^+$, Se$^-$, Si$^-$\} \\
                          & Group 6 & \{B, Si\} \\
\bottomrule
\end{tabular}
\end{small}
\end{table}

\subsection{Full experimental results with standard deviations}
\label{app:full_stats_results}
For statistical significance, we run 3 experiments for every experiment. We put the result including standard deviation of unconditional MOSES generation in Table~\ref{tab:moses_comprehensive_fullstats}, GuacaMol experiment in Table~\ref{tab:guacamol_fullstats}, multi-property guided generation result in Table~\ref{tab:multiproperty_fullstats}, and scaffold extension result in Table~\ref{tab:scaffold_experiment_fullstats}.

\begin{table*}[t!]
\centering
\caption{Unconditional generation on MOSES dataset with full statistics. We bring the reported value from Cometh and DeFoG from their work.
}
\label{tab:moses_comprehensive_fullstats}
\resizebox{\textwidth}{!}{
\renewcommand{\arraystretch}{1.1}
\renewcommand{\tabcolsep}{3.5pt}
\small
\begin{tabular}{l l | c | c c | c c c c c c c}
\toprule
Category & Model  & Quality $\uparrow$ & Scaf-Novel $\uparrow$ & Scaf-Ret. $\uparrow$ & Valid $\uparrow$ & Unique $\uparrow$ & Novel $\uparrow$ & Filters $\uparrow$ & FCD $\downarrow$ & SNN $\uparrow$ & Scaf $\uparrow$ \\

\midrule
- & Training set & 95.4 & \textemdash & \textemdash & 100.0 & 100.0 & \textemdash & 100.0 & 0.48 & 0.59 & - \\

\midrule
1D Sequence & VAE\tabcite{kingma2013auto}  
& 92.8 {\scriptsize $\pm$ 0.2} & 0.22 {\scriptsize $\pm$ 0.01} & 0.031 {\scriptsize $\pm$ 0.003}
& 97.7 {\scriptsize $\pm$ 0.1} & 99.7 {\scriptsize $\pm$ 0.0} & 69.5 {\scriptsize $\pm$ 0.6} & \textbf{99.7 {\scriptsize $\pm$ 0.0}} & 0.57 {\scriptsize $\pm$ 0.00} & 0.58 {\scriptsize $\pm$ 0.01} & 5.9 {\scriptsize $\pm$1.0} \\

& CharRNN~\tabcite{segler2018generating}   
& 92.6 {\scriptsize $\pm$ 2.5} & 0.29 {\scriptsize $\pm$ 0.04} & \textbf{0.035 {\scriptsize $\pm$ 0.003}}   
& 97.5 {\scriptsize $\pm$ 2.6} & 99.9 {\scriptsize $\pm$ 0.0} & 84.2 {\scriptsize $\pm$ 5.1} & 99.4 {\scriptsize $\pm$ 0.3} & \textbf{0.52 {\scriptsize $\pm$ 0.03}} & 0.56 {\scriptsize $\pm$ 0.01} & 11.0 {\scriptsize $\pm$ 0.8} \\

& SAFE-GPT \tabcite{noutahi2024gotta} 
& 92.8 {\scriptsize $\pm$ 0.0} & 0.12 {\scriptsize $\pm$ 0.00} & 0.015 {\scriptsize $\pm$ 0.000}   
& \textbf{99.8 {\scriptsize $\pm$ 0.0}} & 98.9 {\scriptsize $\pm$ 0.0} & 43.7 {\scriptsize $\pm$ 0.3} & 97.7 {\scriptsize $\pm$ 0.0} & 0.72 {\scriptsize $\pm$ 0.02} & 0.57 {\scriptsize $\pm$ 0.01} & 6.3 {\scriptsize $\pm$ 0.7} \\

& GenMol \tabcite{lee2025genmol}   
& 62.1 {\scriptsize $\pm$ 0.0} & 0.05 {\scriptsize $\pm$ 0.00} & 0.012 {\scriptsize $\pm$ 0.001}   
& 99.7 {\scriptsize $\pm$ 0.1} & 64.0 {\scriptsize $\pm$ 0.5} & 68.9 {\scriptsize $\pm$ 0.4} & 98.1 {\scriptsize $\pm$ 0.1} & 16.36 {\scriptsize $\pm$ 0.07} & \textbf{0.64 {\scriptsize $\pm$ 0.01}} & 1.6 {\scriptsize $\pm$ 0.1} \\

\midrule
2D Graph 
& DiGress~\tabcite{vignac2022digress} 
& 82.5 {\scriptsize $\pm$ 0.7} & 0.26 {\scriptsize $\pm$ 0.00} & 0.031 {\scriptsize $\pm$ 0.000}   
& 87.1 {\scriptsize $\pm$ 0.9} & \textbf{100.0 {\scriptsize $\pm$ 0.0}} & 94.2 {\scriptsize $\pm$ 0.2} & 97.5 {\scriptsize $\pm$ 0.0} & 1.25 {\scriptsize $\pm$ 0.03} & 0.53 {\scriptsize $\pm$ 0.00} & 12.8 {\scriptsize $\pm$ 1.4} \\

& DisCo~\tabcite{xu2024discrete}   
& - & - & - 
& 88.3 & 100.0 & \textbf{\underline{97.7}} & 95.6  & 1.44  & 0.50  & 15.1  \\

& Cometh~\tabcite{siraudin2024cometh}   
& 82.1 {\scriptsize $\pm$ 0.1} & 0.36 {\scriptsize $\pm$ 0.00} & 0.023 {\scriptsize $\pm$ 0.000} 
& 87.2 {\scriptsize $\pm$ 0.0} & 100.0 {\scriptsize $\pm$ 0.0} & 96.4 {\scriptsize $\pm$ 0.1} & 97.3 {\scriptsize $\pm$ 0.0} & 1.44 {\scriptsize $\pm$ 0.02} & 0.51 {\scriptsize $\pm$ 0.00} & \textbf{\underline{16.8 {\scriptsize $\pm$ 0.7}}} \\

& DeFoG~\tabcite{qin2024defog}   
& 88.5 {\scriptsize $\pm$ 0.0} & 0.26 {\scriptsize $\pm$ 0.00} & 0.031 {\scriptsize $\pm$ 0.000}   
& 92.8 {\scriptsize $\pm$ 0.0} & 99.9 {\scriptsize $\pm$ 0.0} & 92.1 {\scriptsize $\pm$ 0.0} & \underline{98.9 {\scriptsize $\pm$ 0.0}} & 1.95 {\scriptsize $\pm$ 0.00} & 0.55 {\scriptsize $\pm$ 0.00} & 14.4 {\scriptsize $\pm$ 0.0} \\
\cmidrule{2-12}
& \textbf{MolHIT}   
& \textbf{\underline{94.2 {\scriptsize $\pm$ 0.2}}} & \textbf{\underline{0.39 {\scriptsize $\pm$ 0.00}}} & \underline{0.033 {\scriptsize $\pm$ 0.001}} 
& \underline{99.1 {\scriptsize $\pm$ 0.0}} & 99.8 {\scriptsize $\pm$ 0.0} & 91.4 {\scriptsize $\pm$ 0.2} & 98.0 {\scriptsize $\pm$ 0.00} & \underline{1.03 {\scriptsize $\pm$ 0.02}} & \underline{0.55 {\scriptsize $\pm$ 0.00}} & 14.4 {\scriptsize $\pm$ 1.0} \\
\bottomrule
\end{tabular}}
\end{table*}

\begin{table}[ht]
\centering
\caption{Full statistics of GuacaMol benchmark results (unfiltered). Val.: Validity, V.U.: Unique, V.U.N.: Novel. All results are averaged over 3 runs.}
\label{tab:guacamol_fullstats}
\vspace{-0.05in}
\resizebox{0.5\columnwidth}{!}{ 
\renewcommand{\arraystretch}{1.0}
\renewcommand{\tabcolsep}{3pt} 
\footnotesize 
\begin{tabular}{l ccccc}
\toprule
Model & Val. $\uparrow$ & V.U. $\uparrow$ & V.U.N. $\uparrow$ & KL $\uparrow$ & FCD $\downarrow$ \\
\midrule
Training set & 100.0 & 100.0 & \textemdash & 99.9 & 92.8 \\
\midrule
DiGress (org.) & 85.2 {\scriptsize} & 85.2 {\scriptsize } & 85.1  & 92.9 & \textbf{68.0}  \\
DiGress (full) & 74.7 {\scriptsize $\pm$ 0.4} & 74.6 {\scriptsize $\pm$ 0.5} & 74.0 {\scriptsize $\pm$ 0.4} & 92.4 {\scriptsize $\pm$ 0.5} & 61.1 {\scriptsize $\pm$ 0.2} \\
DiGress+DAE    & 65.2 {\scriptsize $\pm$ 0.4} & 65.2 {\scriptsize $\pm$ 0.4} & 64.9 {\scriptsize $\pm$ 0.4} & 87.0 {\scriptsize $\pm$ 0.4} & 49.2 {\scriptsize $\pm$ 0.6} \\
\midrule
\textbf{MolHIT (Ours)} & \textbf{87.1 {\scriptsize $\pm$ 0.5}} & \textbf{87.1 {\scriptsize $\pm$ 0.3}} & \textbf{86.0 {\scriptsize $\pm$ 0.5}} & \textbf{96.7 {\scriptsize $\pm$ 0.1}} & {54.9 {\scriptsize $\pm$ 0.2}} \\
\bottomrule
\end{tabular}
}
\vspace{-0.1in}
\end{table}

\begin{table*}[t]
\caption{Full statistics of multi-property guided generation on MOSES with 4 different conditions. We report mean absolute error (MAE; $\downarrow$), Pearson correlation ($r$; $\uparrow$), and validity. Avg. is the macro-average across properties. \textbf{Bold} denotes best values.}
\label{tab:multiproperty_fullstats}
\centering
\resizebox{\textwidth}{!}{%
\begin{tabular}{l
  c c c c c
  c c c c c
  c
}
\toprule
& \multicolumn{5}{c}{MAE $\downarrow$} & \multicolumn{5}{c}{Pearson $r \uparrow$} & {Validity (\%) $\uparrow$} \\
\cmidrule(r){2-6} \cmidrule(r){7-11} \cmidrule(l){12-12}
Method
& {QED} & {SA} & {LogP} & {MW} & {\textbf{Avg.}}
& {QED} & {SA} & {LogP} & {MW} &{\textbf{Avg.}}
& {} \\
\midrule
Marginal   
& \mstd{0.117}{0.002} & \mstd{0.115}{0.003} & \mstd{0.067}{0.001} & \mstd{0.272}{0.009} & \mstd{0.143}{0.004}
& \mstd{0.489}{0.003} & \mstd{0.570}{0.012} & \mstd{0.802}{0.003} & \mstd{0.396}{0.001} & \mstd{0.564}{0.005}
& \mstd{75.03}{0.74} \\
Marginal + DAE  
& \mstd{0.107}{0.001} & \mstd{0.094}{0.001} & \mstd{0.061}{0.000} & \mstd{0.227}{0.004} & \mstd{0.122}{0.001}
& \mstd{0.565}{0.005} & \mstd{0.559}{0.009} & \mstd{0.836}{0.005} & \mstd{0.437}{0.015} & \mstd{0.599}{0.002}
& \mstd{87.85}{0.46} \\
\textbf{MolHIT}  
& \bfseries \mstd{0.061}{0.001} & \bfseries \mstd{0.040}{0.001} & \bfseries \mstd{0.049}{0.001} & \bfseries \mstd{0.081}{0.005} & \bfseries \mstd{0.058}{0.002}
& \bfseries \mstd{0.804}{0.009} & \bfseries \mstd{0.790}{0.011} & \bfseries \mstd{0.950}{0.004} & \bfseries \mstd{0.685}{0.024} & \bfseries \mstd{0.807}{0.011}
& \bfseries \mstd{96.31}{0.23} \\
\bottomrule
\end{tabular}
}
\end{table*}

\begin{table}[t!]
\centering
\caption{Full statistics of scaffold extension results on MOSES. Results are averaged over 3 runs of 10,000 targets. \textbf{Hit@k} denotes the recovery of ground-truth within $k$ samples.}
\label{tab:scaffold_experiment_fullstats}
\vspace{-0.05in}
\resizebox{0.5\columnwidth}{!}{
\renewcommand{\arraystretch}{1.1}
\renewcommand{\tabcolsep}{2.5pt} 
\footnotesize
\begin{tabular}{l cccc} 
\toprule
Model & Valid (\%) $\uparrow$ & Diversity $\uparrow$ & Hit@1 $\uparrow$ & Hit@5 $\uparrow$ \\
\midrule
DiGress & 50.8 {\scriptsize $\pm$ 0.5} & 44.8 {\scriptsize $\pm$ 1.8} & 2.07 {\scriptsize $\pm$ 0.09} & 6.41 {\scriptsize $\pm$ 0.21} \\
Marginal + DAE & 64.8 {\scriptsize $\pm$ 0.2} & \textbf{58.0 {\scriptsize $\pm$ 0.1}} & 1.67 {\scriptsize $\pm$ 0.10} & 6.37 {\scriptsize $\pm$ 0.24} \\
\midrule
\textbf{MolHIT (Ours)} & \textbf{83.9 {\scriptsize $\pm$ 0.4}} & 57.4 {\scriptsize $\pm$ 0.6} & \textbf{3.92 {\scriptsize $\pm$ 0.23}} & \textbf{9.79 {\scriptsize $\pm$ 0.09}} \\
\bottomrule
\end{tabular}
}
\vspace{-0.1in}
\end{table}

\subsection{Implementation of baselines}
For all baselines, we use released checkpoints when available. Otherwise, we train the models using their official codebase, following the training hyperparameters reported in the paper or provided in the codebase.
Note that the original GenMol~\citep{lee2025genmol} model was trained on a much larger molecule dataset, so we train the model on the MOSES dataset for a fair comparison.

\subsection{Unconditional generation with MOSES and GuacaMol}
\label{app:exp_details_uncond_moses}

\paragraph{Training details}
For our model backbone, we adopt the graph transformer proposed by \citet{vignac2022digress}, which simultaneously predicts node and edge features. To ensure a fair comparison across all experimental settings, we maintain a consistent architecture of 12 transformer blocks without altering any internal dimensional configurations. The total trainable parameter count is approximately 16.2M. The introduction of additional token indices (DAE and HDDM) adds negligible overhead where representing a variance of less than $0.01\%$ in total parameters. For training stability, we employ gradient clipping with a threshold of 2.0 and an Exponential Moving Average (EMA) rate of 0.999. We early stop with 100 epoch training with MOSES and 50 epochs with GuacaMol, compared to the original 300 epoch training of other graph diffusion baselines~\citep{vignac2022digress,siraudin2024cometh,qin2024defog}. We also remove calculating geometric prior originally used in \citet{vignac2022digress}, where they use extra graph features as conditional information. In our experiments, this has negligible effects on the performance.

\paragraph{Evaluation of MOSES}
\label{app:moses_evaluation_metrics}

The following metrics are utilized to evaluate the generative performance on the MOSES dataset, following the standardized protocols established by \citet{polykovskiy2020molecular}.

\begin{itemize}\item \textbf{Validity ($\uparrow$):} The fraction of generated molecules that pass RDKit's sanitization checks and basic chemical valency rules. High validity is a primary indicator that the \textbf{DAE} system successfully constrains the sampling process to the chemically feasible manifold.\item \textbf{Uniqueness ($\uparrow$):} The proportion of valid molecules that are not duplicates. This measures the model's ability to avoid mode collapse and explore a diverse structural space.

\item \textbf{Novelty ($\uparrow$):} The fraction of valid, unique molecules that were not present in the training set. This differentiates between a model that has memorized the data and one that has learned the underlying generative distribution.

\item \textbf{Filters ($\uparrow$):} The percentage of generated molecules that pass common medicinal chemistry filters (e.g., MCULE, BRENK, and PAINS). This evaluates the drug-likeness and synthetic viability of the generated samples.

\item \textbf{Fr\'{e}chet ChemNet Distance (FCD~\citep{preuer2018frechet}, $\downarrow$):} A measure of the distance between the multivariate distributions of generated and test molecules in the feature space of ChemNet. FCD captures both chemical and biological similarity, serving as the most rigorous metric for distributional fidelity.

\item \textbf{Similarity to Nearest Neighbor (SNN, $\uparrow$):} The average Tanimoto similarity between a generated molecule and its closest neighbor in the test set. High SNN indicates that the model has captured the specific structural motifs and chemical space of the benchmark.

\item \textbf{Scaffold Similarity (Scaf, $\uparrow$):} The cosine similarity between the frequencies of Bemis--Murcko scaffolds \citep{bemis1996properties} in the generated and test sets. This assesses whether the model's learned distribution of backbone structure matches the architectural diversity of real-world leads.
\end{itemize}

\paragraph{Baselines}

\subsection{Structure novelty metric}
\label{app:structure_novelty_metric}

\begin{itemize}

\item \textbf{Scaffold Novelty:} We evaluate the model's ability to innovate beyond the training distribution using Bemis–Murcko scaffolds~\citep{bemis1996properties}. The absolute number of unique generated scaffolds absent from the training set: 
\begin{equation} \text{Scaf-Novel} = \frac{|\mathcal{S}_{\text{gen}} \setminus \mathcal{S}_{\text{train}}|}{n_{\text{total}}}. \end{equation} This metric quantifies the model's capacity for structural extrapolation, measuring its efficiency in exploring the beyond the molecular structure of the given dataset.

\item \textbf{Scaffold Retrieval:} This assesses the model's ability to rediscover known, high-quality frameworks from the held-out test set. This is defined as the absolute number of unique test-set scaffolds successfully generated: 
from the training set: 
\begin{equation}
    \text{Scaf-Ret} = \frac{|\mathcal{S}_{\text{gen}} \cap \mathcal{S}_{\text{test}}|}{n_{\text{total}}}.
\end{equation}
Scaffold retrieval serves as a rigorous test of distributional accuracy. A high retrieval density demonstrates that the model has not merely learned to generate novel-looking noise, but has accurately captured the underlying manifold of valid, drug-like molecules defined by the test distribution.
\end{itemize}

\subsection{Unconditional generation with GuacaMol}
\label{app:exp_detail_uncond_guacamol}

\paragraph{GuacaMol experiment}
For experiment on GuacaMol, we test our algorithm on the unfiletered, full dataset. Previous graph diffusion model baselines~\citep{vignac2022digress, siraudin2024cometh, qin2024defog} train the model on the filtered dataset, where they filter out the molecules that are failed to be reconstructed back. This can bias the training data distribution. In contrast, we use full, unfiltered dataset for the experiment and since there is no graph diffusion baseline, we compare with the original DiGress trained on a full GuacaMol dataset with coarse atom encoding, Discrete Diffusion using marginal transition with DAE, and compare them with MolHIT.

\subsection{Multi-property guided generation}
\label{app:exp_details_multiprop}

\paragraph{Data construction}
For conditional generation, we augment the MOSES dataset \citep{polykovskiy2020molecular} with four continuous molecular descriptors: Quantitative Estimate of Drug-likeness (QED), Synthetic Accessibility (SA) score, Octanol-Water Partition Coefficient ($\text{logP}$), and Molecular Weight (MW).

\begin{itemize}
    \item \textbf{Quantitative Estimate of Drug-likeness (QED, $\uparrow$):}
    A widely used composite score that summarizes multiple molecular properties (e.g., lipophilicity, polarity, and molecular size) into a single measure of drug-likeness; higher values indicate more drug-like compounds.

    \item \textbf{Synthetic Accessibility (SA, $\downarrow$):}
    An empirical estimate of synthetic difficulty that combines fragment-based contributions with a complexity penalty; lower values indicate molecules that are easier to synthesize.

    \item \textbf{Octanol--Water Partition Coefficient (logP):}
    A measure of lipophilicity that is informative of solubility and membrane permeability; excessively high logP is typically associated with poor solubility and unfavorable ADMET profiles.

    \item \textbf{Molecular Weight (MW):}
    The molecular mass in Daltons. Consistency with the training distribution (e.g., MOSES) helps ensure generated molecules remain within a drug-like regime.
\end{itemize}

All properties are calculated using the RDKit library and the sascorer module~\citep{ertl2009estimation}. To ensure stable convergence of the conditioning vector $\mathbf{C}$ within our AdaLayerNorm layers, we perform min-max normalization on these values using the global statistics of the training split, which are in Table~\ref{tab:multi_property_norm_stats}.  

\begin{table}[ht]
\centering
\caption{Min and max values for molecular property conditioning in MOSES training / test split.}
\label{tab:multi_property_norm_stats}
\begin{small}
\begin{tabular}{ll c c c c}
\toprule
\textbf{Split} & \textbf{Statistic} & \textbf{QED} & \textbf{SA Score} & \textbf{logP} & \textbf{MW (Da)} \\
\midrule
\multirow{2}{*}{Training} & Min & 0.1912 & 1.2694 & -5.3940 & 250.017 \\
                          & Max & 0.9484 & 7.4831 & 5.5533 & 349.999 \\
\midrule
\multirow{2}{*}{Test}     & Min & 0.2265 & 1.3339 & -4.2894 & 250.042 \\
                          & Max & 0.9484 & 6.6916 & 5.7255 & 349.990 \\
\bottomrule
\end{tabular}
\end{small}
\end{table}

\paragraph{Conditional graph transformer}
\label{app:exp_details_cond_graph_transformer}

While we maintain the core node-edge attention mechanism of the original graph transformer~\citep{vignac2022digress}, we introduce several key modifications to enable conditional modeling. First, we remove the persistent global feature vector $y$—which in the original framework is updated at every layer—and replace it with a centralized conditioning vector $\mathbf{C}$. This vector is composed of a sinusoidal timestep embedding \citep{ho2020denoising} and an optional MLP-encoded external property condition $\mathbf{c}$. Second, to integrate $\mathbf{C}$ into the denoising process, we replace standard Layer Normalization with Adaptive Layer Normalization (AdaLayerNorm) for node features. Specifically, for a node embedding $\mathbf{x}$, the normalization is defined as:
$$\text{AdaLN}(\mathbf{x}, \mathbf{C}) = (1 + \gamma(\mathbf{C})) \cdot \text{LayerNorm}(\mathbf{x}) + \beta(\mathbf{C})$$
where $\gamma$ and $\beta$ are affine transformations of the conditioning vector. This allows the global context (time and properties) to directly modulate the scale and shift of node representations. Finally, we implement Classifier-Free Guidance (CFG) support by incorporating a dropout mechanism on the property embedding during training, while ensuring the temporal signal remains persistent to maintain denoising stability. Our conditional graph transformer naturally inherits permutation equivariance, which is different from the~\citet{liu2024graph}.

\paragraph{Evaluation details}
For sampling, we employ Classifier-Free Guidance (CFG) \citep{ho2022classifier} with a guidance scale of $w=1.0$. We observe that in our discrete graph-diffusion framework, increasing the guidance weight beyond unity did not consistently yield better property alignment. We leave the better design the sampler or models to be effective in higher guidance strength $w$ as a promising avenue for future research.

\subsection{Scaffold extension}
\label{app:exp_scaffold_extension}

\paragraph{Task Formulation} Given a ground-truth molecule $\mathcal{G}$ from the test split, we use RDKit to extract its scaffold $\mathcal{S} \subset \mathcal{G}$. The task is to generate a completed molecule $\mathcal{M}$ such that $\mathcal{S} \subseteq \mathcal{M}$. To isolate the generative capability from size prediction errors, we bound the generation size (number of atoms) to match $|\mathcal{G}|$.

\paragraph{Sampling Protocol}
At each reverse timestep $t$, the region corresponding to $\mathcal{S}$ is forced to be the same (i.e, $q(\mathbf{x}_t | \mathbf{x}_{\text{scaffold}}) =\mathbf{x}_{\text{scaffold}} $, ensuring the scaffold region is strictly fixed during the generation. The extension region is initialized from the limit distribution (i.e, prior) $p_{\text{prior}}$ and evolved via the standard reverse process. We generate $K=1,5$ independent samples per scaffold to capture the model's exploration capability. 
\paragraph{Metric Definitions} Metrics are computed per scaffold and then averaged across the test set. Let $\mathcal{M}_1, \dots, \mathcal{M}_K$ be the generated graphs for a single scaffold.
\begin{itemize}\item \textbf{Validity:} The fraction of $\mathcal{M}_i$ that are chemically valid according to RDKit sanitization.

\item \textbf{Diversity:} Calculated on the unique valid set $\mathcal{U}$. We define diversity as $1 - \frac{1}{|\mathcal{U}|^2} \sum_{u, v \in \mathcal{U}} \text{Sim}(u, v)$, where $\text{Sim}$ is the Tanimoto similarity using Morgan fingerprints ($r=2$, 2048 bits).
\item \textbf{Exact Match (Hit@$K$):} A binary indicator, set to 1 if the ground truth $\mathcal{G}$ (canonical SMILES) is present in the generated set $\{\mathcal{M}_1, \dots, \mathcal{M}_K\}$, and 0 otherwise.

\end{itemize}


\end{document}